\documentclass[journal,onecolumn]{IEEEtran}
\usepackage{algorithmic}
\usepackage{algorithm}
\usepackage{array}
\usepackage[caption=false,font=normalsize,labelfont=sf,textfont=sf]{subfig}
\usepackage{textcomp}
\usepackage{stfloats}
\usepackage{url}
\usepackage{verbatim}
\usepackage{cite}
\usepackage[ngerman,main=english]{babel}
\usepackage{graphicx}
\usepackage{framed}
\usepackage[normalem]{ulem}
\usepackage{amsmath}
\usepackage{amsthm}
\usepackage{amssymb}
\usepackage{mathtools}
\usepackage{amsfonts}
\usepackage{xfrac}
\usepackage{enumerate}
\usepackage{comment}
\usepackage[utf8]{inputenc}
\usepackage[top=1 in,bottom=1in, left=1 in, right=1 in]{geometry}

\usepackage{hyperref} 
\usepackage{cleveref}

\usepackage{xcolor}

\usepackage{tikz}
\usepackage{pgf}
\usetikzlibrary{shapes.geometric, arrows, mindmap,shapes,positioning,backgrounds,decorations}
\usepackage{pgfplots}
\pgfplotsset{width=10cm,compat=1.9}
\usepgfplotslibrary{colormaps,fillbetween}

\theoremstyle{definition}
\newtheorem{Definition}{Definition}[section]
\newtheorem*{Example}{Example}

\theoremstyle{plain}
\newtheorem{Theorem}[Definition]{Theorem}
\newtheorem{Corollary}[Definition]{Corollary}
\newtheorem{Lemma}[Definition]{Lemma}

\theoremstyle{remark}
\newtheorem{Remark}[Definition]{Remark}

\usepackage[autostyle]{csquotes}

\newcommand{\R}{\mathbb{R}}
\newcommand{\Q}{\mathbb{Q}}

\newcommand{\N}{\mathbb{N}}

\newcommand{\C}{\mathbb{C}}

\newcommand{\abs}[1]{\ensuremath{\left\lvert#1\right\rvert}}
\newcommand{\aabs}[1]{\ensuremath{\lvert#1\rvert}}
\newcommand{\aaabs}[1]{\ensuremath{\Big\lvert#1\Big\rvert}}
\newcommand{\norm}[2][]{\ensuremath{\left\lVert#2\right\rVert_{#1}}}
\newcommand{\nnorm}[2][]{\ensuremath{\lVert#2\rVert^{}_{#1}}}
\newcommand{\nnormsquaared}[2][]{\ensuremath{\lVert#2\rVert^{2}_{#1}}}
\newcommand{\nnnorm}[2][]{\ensuremath{\Big\lVert#2\Big\rVert_{#1}}}
\DeclareMathOperator*{\argmin}{arg\,min}
\DeclareMathOperator*{\dom}{dom}

\makeatletter
\def\ps@IEEEtitlepagestyle{%
  \def\@oddfoot{\mycopyrightnotice}%
  \def\@oddhead{\hbox{}\@IEEEheaderstyle\leftmark\hfil\thepage}\relax
  \def\@evenhead{\@IEEEheaderstyle\thepage\hfil\leftmark\hbox{}}\relax
  \def\@evenfoot{}%
}
\def\mycopyrightnotice{%
  \begin{minipage}{\textwidth}
  \centering \scriptsize
  \copyright~20XX IEEE.  Personal use of this material is permitted.  Permission from IEEE must be obtained for all other uses, in any current or future media, including reprinting/republishing this material for advertising or promotional purposes, creating new collective works, for resale or redistribution to servers or lists, or reuse of any copyrighted component of this work in other works.
  \end{minipage}
}
\makeatother

\begin{document}

\title{Limitations of Deep Learning for Inverse Problems on Digital Hardware}

\author{Holger Boche, \IEEEmembership{Fellow, IEEE}, Adalbert Fono and Gitta Kutyniok, \IEEEmembership{Senior Member, IEEE}
\thanks{This work of Holger Boche was supported in part by the German Federal Ministry of Education and Research (BMBF) in the project Hardware Platforms and Computing Models for Neuromorphic Computing (NeuroCM) under Grant 16ME0442 and within the national initiative on 6G Communication Systems through the Research Hub 6G-life under Grant 16KISK002.}
\thanks{This work of Gitta Kutyniok was supported in part by the Konrad Zuse School of Excellence in Reliable AI (DAAD), the Munich Center for Machine Learning (BMBF) as well as the German Research Foundation under Grants DFG-SPP-2298, KU 1446/31-1 and KU 1446/32-1. Furthermore, G. Kutyniok acknowledges support from LMUexcellent, funded by the Federal Ministry of Education and Research (BMBF) and the Free State of Bavaria under the Excellence Strategy of the Federal Government and the Länder as well as by the Hightech Agenda Bavaria.}
\thanks{Holger Boche is with the Institute of Theoretical Information Technology, Technical University of Munich, Munich, Germany, and also with the BMBF Research Hub 6G-life, with the Excellence Cluster Cyber Security in the Age of Large-Scale Adversaries, Ruhr University Bochum, Bochum, Germany, and also with the Munich Center for Quantum Science and Technology (MCQST), Munich, Germany, and also with the Munich Quantum Valley (MQV), Munich, Germany.}
\thanks{Adalbert Fono is with the Department of Mathematics, Ludwig-Maximilians-Universität München, Germany.}
\thanks{Gitta Kutyniok is with the Department of Mathematics, Ludwig-Maximilians-Universität München, Germany, and also with the Munich Center for Machine Learning (MCML), Munich, Germany.}
}

\markboth{IEEE TRANSACTIONS ON Information Theory}%
{Boche, Fono and Kutyniok: Limitations of Deep Learning for Inverse Problems on Digital Hardware}

\maketitle

\begin{abstract}
Deep neural networks have seen tremendous success over the last years. Since the training is performed on digital hardware, in this paper, we analyze what actually can be computed on current hardware platforms modeled as Turing machines, which would lead to inherent restrictions of deep learning. For this, we focus on the class of inverse problems, which, in particular, encompasses any task to reconstruct data from measurements. 
We prove that finite-dimensional inverse problems are not Banach-Mazur computable for small relaxation parameters.  
Even more, our results introduce a lower bound on the accuracy that can be obtained algorithmically. 
\end{abstract}


\section{Introduction}

\IEEEPARstart{T}{he} solution of problems which do not admit a closed form solution is often a computationally tremendously challenging task. In fact, for many problems closed form solutions are not only not known, but can provably not exist \cite{poonen14}. Therefore, general algorithms for solving problems such as Diophantine equations are often not known. In 1900, Hilbert proposed the problem of finding an algorithm which, for any polynomial equation with integer coefficients and a finite number of unknowns \cite{Hilbert00Problems}, decides whether the polynomial possesses integer-valued roots. This problem, commonly known as Hilbert's tenth problem, turned out to be unsolvable --- there does not exist such an algorithm \cite{Matiyasevich70Diophantine}. Further examples include physical processes that are often modeled mathematically in the form of differential equations in order to study their properties, even without exact knowledge of their solutions. This calls for approaches which allow to make profound observations about these processes without explicit access to closed form and exact algorithmic solutions. 

A practical alternative in this regard is to simulate the underlying systems and observe the evolution of the model over time to gain insight into its functioning. The simulated models are often of a continuum nature, i.e., their input and output quantities are described by continuous variables. At first glance, it seems natural to solve or simulate continuous problems on analog computers. Nonetheless, the tremendous success of digital computing and the conviction that any continuous problem can be simulated on digital machines led to the present situation that computations are performed almost exclusively on digital hardware. The question whether a continuous system can be simulated on digital computers lies therefore at the heart of the foundations of signal processing and computer science.

\subsection{Deep Learning}
A completely different approach to tackle various problems has recently emerged in the form of deep learning. Instead of explicitly modelling a process, its goal is to learn the underlying relations by using an \textit{(artificial) neural network}. Inspired by the functionality of the human brain, McCulloch and Pitts \cite{McCulloch43NNs} introduced the idea of a neural network in 1943. Roughly speaking, a neural network consists of neurons arranged in layers and connected by weighted edges; in mathematical terms we can think of a concatenation of affine-linear functions and relatively simple non-linearities.

Triggered by the drastic improvements in computing power of digital computers and the availability of vast amounts of training data, the area has seen overwhelming practical progress in the last fifteen years. Deep neural networks (which in fact inspired the name "deep learning"), i.e., networks with large numbers of layers, have been applied with great success in many different fields, even yielding state of the art results in applications such as image classification \cite{He2015DelvingDI}, playing board games \cite{Silver16Go}, natural language processing \cite{Brown20GPT3}, and protein folding prediction \cite{Senior20DeepFold}. Interestingly, the most significant progress and superhuman abilities seem to occur on discrete problems such as the aforementioned board games including chess and Go. For an in-depth overview, we refer to the survey paper by LeCun, Bengio, and Hinton \cite{LeCun15DL}.

A neural network implements a non-linear mapping and the primary goal is to approximate an unknown function based on a given set (of samples) of input-output value pairs. This is typically accomplished by adjusting the network’s biases and weights according to an optimization process; the standard approach so far is stochastic gradient descent (via backpropagation) \cite{Rumelhart86BP}. The optimization step is usually referred to as the training of a neural network.

An active field of research, which was to a certain extent swept by deep learning approaches over the last several years, is the area of inverse problems in imaging sciences. Image reconstruction from measurements is an important task in a wide range of scientific, industrial, and medical applications such as electron microscopy, seismic imaging, magnetic resonance imaging (MRI) and X-Ray computed tomography (CT). Numerous contributions to solving such inverse problems using deep neural networks could be recently witnessed, claiming to achieve competitive or even superior performance compared to current techniques.

A unified framework for image reconstruction by manifold approximation as a data-driven supervised learning task is proposed in \cite{Zu18AutoMap}. In \cite{Arridge2019SolvingIP}, the authors survey methods solving ill-posed inverse problems combining data-driven models, and in particular those based on deep learning, with domain-specific knowledge contained in physical–analytical models. Approaches to tackle specific inverse problems have been presented for instance in \cite{Bubba19Shearlet} for limited angle CT, in \cite{Yang16MRIDL, Hammernik18MRIDL2} for MRI, in \cite{Chen2018LowLightPhoto} for low-light photography, in \cite{Rivenson17DLmicroscopy} for computational microscopy, and in \cite{Araya18DLtomography} for geophysical imaging.    

Despite their impressive success, it is important to emphasize that neural networks share inherent restrictions with digital computers. Today's neural networks operate on digital hardware, therefore the same limitations concerning continuous problems do apply. More precisely, although neural networks are used seemingly successful on myriads of problems, it is crucial to analyze whether the given problem and input data do in fact allow for an effective computation in the sense that every computed output comes with a guaranteed worst-case error bound. The importance of this question should not be underestimated, since we may unknowingly obtain erroneous outputs for certain problems or specific inputs even if neural networks demonstrated their power for other problems or inputs. In this paper we will analyze this obstacle for a large class of inverse problems, which are a canonical model for, in particular, image reconstruction tasks such as CT, MRI, etc.

\subsection{Digital Computations of Continuous Models}
Digital hardware only allows the exact computation of the solution of finite discrete problems. In case of continuous problems, algorithms on digital hardware can only provide an approximation of the exact solution. In contrast to classical approximation theory, where the mere mathematical existence of an approximation is sufficient, the essential point for computability is that the approximation can be {\em algorithmically} and {\em effectively} computed in a finite number of steps. 
The requirement of effectiveness is particularly important, since it demands that the digital system should not only provide an approximate solution of the continuous problem, but it also has to quantify the error between the approximation and the true (but generally unknown) solution of the continuous problem. Even more, the digital system should be able to provide a solution which is guaranteed to lie within any prescribed error bound. Only if such an effective approximation on a digital system is possible, one would say that the continuous problem is computable (on a digital machine); see \Cref{img:UnCompFunc} for an illustration. These requirements are of particular importance if a machine needs to autonomously apply specific operations to unknown input data. Then the computability of the operations is essential to prevent erroneous calculations. Basically any simulation or numerical solution of a problem on a digital computer relies on the fact that the problem itself is computable. Otherwise, the obtained output may be, unknowingly to the user of the digital machine, far away from the sought true solution as worst-case error bounds are missing. A specific use case where such problems arise is the numerical solution of differential equations \cite{Boche20LTI, Weihrauch02WaveProp, Zhong14EffConvDE, Graca21CompDE}. Another important issue is trustworthiness in engineering in 6G applications, whereby trustworthiness can not be reached without reliable, i.e., computable, systems \cite{Fettweis20226G}.

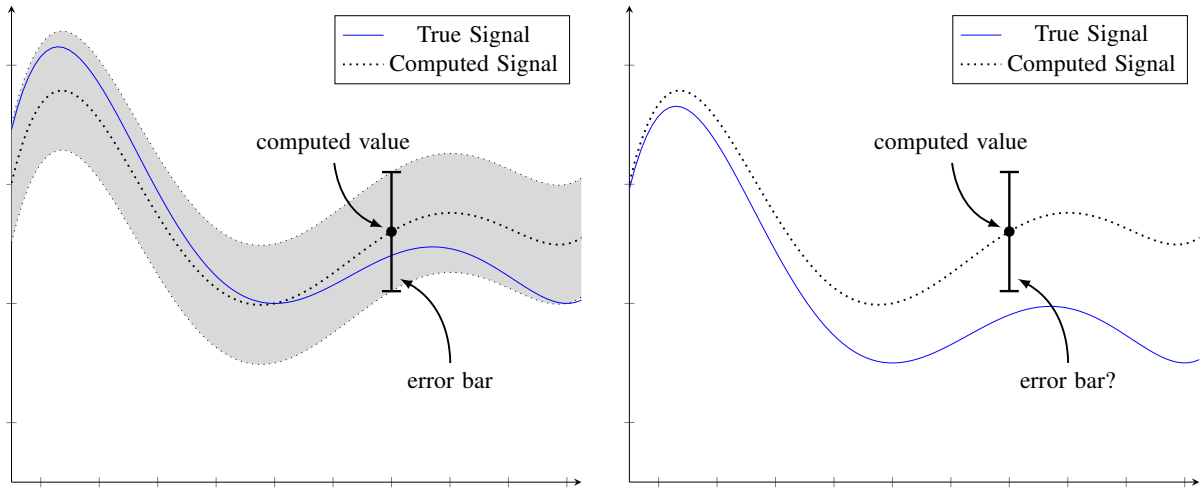
\begin{figure}[ht]
    \centering
\scalebox{0.9}{
\begin{tikzpicture}
\begin{axis}[
    axis lines = left,
    ymin=-0.3, ymax=0.5,
    yticklabel=\empty,
    xticklabel=\empty
]
\addplot [
    domain=-0.9:1.05, 
    samples=100, 
    color=blue,
]
{x^5 -x^4 - x^3 + x^2};
\addlegendentry{True Signal}

\addplot [
    domain=-0.9:1.05, 
    samples=100, 
    dotted,
    thick,
]
{x^5 -x^4 - x^3 + x^2 +0.1*x };
\addlegendentry{Computed Signal}

\addplot [name path=A,black, dotted, domain=-0.9:1.05, samples=100] {x^5 -x^4 - x^3 + x^2 +0.1*x + 0.1 };
\addplot [name path=B,black, dotted, domain=-0.9:1.05, samples=100] {x^5 -x^4 - x^3 + x^2 +0.1*x - 0.1};
\addplot [black!15] fill between[of=A and B];

\addplot[
    color=black,    mark=*, error bars/.cd, y dir=both, y explicit, error bar style={line width=1pt},  error mark options={
      rotate=90, mark size=4pt, line width=1pt}
    ]
    coordinates {
    (0.4,0.12064)  +- (0,0.1)
    };
    
\node[anchor=north] (source) at (axis cs:0.2,0.3){computed value};
\node (destination) at (axis cs:0.4,0.12064){};
\draw[->,-{latex[scale=5]}, thick](source)  to[bend left=-30]  (destination);    

\node[anchor=north] (source2) at (axis cs:0.6,-0.1){error bar};
\node (destination2) at (axis cs:0.4,0.05){};
\draw[->,-{latex[scale=5]}, thick](source2)  to[bend left=-30]  (destination2); 

\end{axis}
\end{tikzpicture}
}
\scalebox{0.9}{
\begin{tikzpicture}
\begin{axis}[
    axis lines = left,
    ymin=-0.3, ymax=0.5,
    yticklabel=\empty,
    xticklabel=\empty
]
\addplot [
    domain=-0.9:1.05, 
    samples=100, 
    color=blue,
]
{x^5 -x^4 - x^3 + x^2 - 0.1};
\addlegendentry{True Signal}

\addplot [
    domain=-0.9:1.05, 
    samples=100, 
    dotted,
    thick,
]
{x^5 -x^4 - x^3 + x^2 +0.1*x };
\addlegendentry{Computed Signal}


\addplot[
    color=black,    mark=*, error bars/.cd, y dir=both, y explicit, error bar style={line width=1pt},  error mark options={
      rotate=90, mark size=4pt, line width=1pt}
    ]
    coordinates {
    (0.4,0.12064)  +- (0,0.1)
    };
    
\node[anchor=north] (source) at (axis cs:0.2,0.3){computed value};
\node (destination) at (axis cs:0.4,0.12064){};
\draw[->,-{latex[scale=5]}, thick](source)  to[bend left=-30]  (destination);    

\node[anchor=north] (source2) at (axis cs:0.6,-0.1){error bar?};
\node (destination2) at (axis cs:0.4,0.05){};
\draw[->,-{latex[scale=5]}, thick](source2)  to[bend left=-30]  (destination2); 

\end{axis}
\end{tikzpicture}
}
\caption{For a computable signal we can always determine an error bar and can then be sure that the true value lies within the specified error range. This situation is depicted on the left. However, for an non-computable signal this guarantee may fail as shown on the right.}
\label{img:UnCompFunc}
\end{figure}
For practical applications the efficiency of an algorithm is of key importance --- not only the correctness, but additionally the required resources such as time, storage, and data are relevant. Classical computer science focusing on discrete problems provides us with suitable notions, for instance, the commonly known classes P and NP constitute a notion of efficiency. Roughly speaking, algorithms in P are considered efficient and applicable in practice whereas algorithms in NP generally do not suffice practical demands. However, before the efficiency of an algorithm can be addressed, one needs to determine the effective solvability first. In this sense, effectiveness is the minimum requirement for any efficient algorithm.

A theoretical model for digital computers is given by Turing machines \cite{Turing36Entscheidung}, which capture the logic of digital computations but neglect real-word limitations like memory constraints, energy consumption, etc. Consequently, Turing machines offer a mathematical model describing the power and limitations of digital computations. Using this model, it has been shown that, if the effectiveness of the approximation is taken into consideration, then not every continuous system can be simulated on a digital machine; examples include the Cauchy problem for the three dimensional wave equation \cite{PourEl97WaveEq}, channel capacities in information theory \cite{Elkouss2018MemoryEC, Schaefer2019TuringMS}, and operations in signal processing such as spectral factorization \cite{Boche20SpecFac} or the Fourier transform of discrete and bandlimited signals \cite{Boche20BandlimitedSignals}. In fact, even very simple analog systems such as stable, linear, and time-invariant systems can generally not be simulated on a digital machine \cite{Boche20LTI}. Only for certain subsets of all admissible input signals can these continuous models be effectively approximated on digital computers. For these reasons, it is important that users of digital simulation software are able to recognize whether their problem and the given input data allow for an effective computation.

The previous discussion shows that it is crucial to be aware of circumstances where the output of a (physical) system can generally not be calculated on a digital machine. In such a situation, we may still ask whether there exist subsets of input signals such that the output of the system can be predicted on a digital machine by controlling algorithmically the approximation error. This leads to the question of determining classes of “good” input signals which are as large as possible. We may even ask, if it is possible to have a machine which is able to decide autonomously whether, for a given input signal, the output can be effectively approximated.

\subsection{Contributions}
Despite the tremendous success of deep learning in various applications, several downsides have been observed. Particularly problematic for their applicability are, for instance, instabilities due to adversarial examples in image classification \cite{Szegedy14AdvEx} and image reconstruction \cite{Antun2020InstabilitiesDL}, the fact that deep neural networks still act as a black box lacking explainability \cite{Xie20XAI, Kolek2021ratedistortion}, and in general a vast gap between theory and practise (see, for instance, \cite{Adcock20gap, Berner2021modernMathDL}). This led to the acknowledgement that a better understanding as well as a mathematical theory of deep learning is in great demand. In fact, these observations might also point to underlying problems of a computability nature as discussed in the previous section. Therefore, our goal is to study the boundaries of deep learning in the following sense: 
\begin{center}
    \textit{What can actually be computed on digital hardware and what are inherent restrictions of deep learning (performed on digital hardware)?}    
\end{center} 
We focus on a specific scenario - finite-dimensional inverse problems - where (prior to the results of [36] and [38]) arbitrarily accurate neural networks were expected to exist \cite{colbrook21stable}, but are not obtained in practice due to flaws such as instability \cite{gottschling20troublesome}.
Although we consider only a very restricted setting, the underlying question(s) are universal and important in basically any application domain, ranging from industry, over the public up to the sciences. 
\begin{center}
    \textit{Are we missing the correct tools and algorithms to train neural networks adequately on digital machines or do such algorithms not exist at all?}
\end{center}
As shown in \cite{colbrook21stable}, the answer for finite dimensional inverse problems is unfortunately negative. In particular, \cite{colbrook21stable} establishes fundamental limitations of any training algorithm in a variety of computational models, including deterministic and randomized Turing machines by introducing lower bounds on the accuracy of any randomized or deterministic algorithmically trained neural network tackling inverse problems. We supplement these results in the deterministic setting by proving non-approximability in \Cref{thm:nonApprox} of inverse problems in the Banach-Mazur sense (see \Cref{sec:preliminaries Computation} for an overview of computability theory).
Thereby, our focus is on inverse problems tackled via basis pursuit and lasso optimization. In this way, we supplement the findings in \cite{bastounis21extended,colbrook21stable} on the study of limitations of deep learning for solving inverse problems. In particular, starting from the Banach-Mazur framework, we show that there exists a lower bound on the achievable accuracy of any algorithm  which performs the training of a neural network on digital hardware. This reinforces the described limits on the accuracy of algorithmically computed approximations of inverse problems via neural networks.

Summarizing, we provide further evidence on the limitations of neural networks for solving finite-dimensional inverse problems given that the computations are performed on Turing machines (as abstract model of digital computers). In the Banach-Mazur setting, the barrier on the capabilities of neural networks is caused by the following two separate aspects:
\begin{itemize}
    \item The mathematical structure and properties of finite-dimensional inverse problems.
    \item The mathematical structure and properties of Turing machines and thereby also of digital machines.
\end{itemize}
It is important to stress that our negative result is due to a combination of these aspects. However, as digital computations are dominant --- at this time --- in basically any field including deep learning, our findings confirm the previously reported \cite{colbrook21stable, bastounis21extended} profound limitations of current deep learning approaches and is consequently of tremendous practical relevance.

\subsection{Related Work}\label{subsec:relatedWork}
Next, we describe the the results in \cite{bastounis21extended, colbrook21stable} concerning the limits of any algorithmic computations tackling inverse problems in more detail. Moreover, classification problems were also analyzed with a similar methodology \cite{bastounis2021Classification}. Based on the notion of a general algorithm \cite{Ben2015SCI} the existence of algorithms \cite{bastounis21extended}, in particular, training algorithms of neural networks \cite{colbrook21stable}, solving finite-dimensional inverse problems was studied. Thereby, the notion of a general algorithm encompasses deterministic as well as probabilistic algorithms and various computing models such as Turing \cite{Turing36Entscheidung} and Blum–Shub–Smale (BSS) \cite{Blum89BSS} machines. The computing models describe the feasible algorithmic operations, i.e., the operations an algorithm can rely on if executed on said (abstract) machine. One key feature of a general algorithm is the provided input representation. A general algorithm is given a rational sequence $(x_n)_{n\in\N}$ approximating the 'true' input $x\in \C^n$ and an error parameter $\varepsilon >0$. The algorithm is asked to output a number which is at least $\varepsilon$-close to the solution set  --- the solution may not be unique --- associated to $x$. Thereby, an algorithm is successful if it can accomplish this task for any feasible input and error parameter. In contrast, the largest $\varepsilon$ for which all algorithms will fail to provide $\varepsilon$-accuracy is called the breakdown epsilon of the problem.

Thus, the question studied in \cite{bastounis21extended} and \cite{colbrook21stable} is the following: Given arbitrarily accurate approximations to any complex number such as, in particular, the training samples, do there exist general algorithms that yield accurate neural networks tackling inverse problems or are there computational barriers preventing the training of neural networks under these conditions? It turns out that these algorithms can not exist in fairly general circumstances, i.e., for any error parameter $\varepsilon= 10^{-k}$, $k>2$, one can find classes of inputs where any algorithm will fail to compute $\varepsilon$-close solutions. Therefore, also the existence of training algorithms for computing neural networks, which output $\varepsilon$-close solutions on the given input class, is excluded i.e., breakdown epsilon exists and are indeed greater than zero. Moreover, the failure of algorithms is not related to ill-conditioned instances; the constructed input classes contain only well-condition problems.

We consider the same fundamental question concerning the existence and achievable accuracy of algorithms. In particular, we are interested in autonomous digital computations modeled by (deterministic) Turing machines.
Therefore, the starting point of our analysis is the theory of recursive functions and mathematical logic and, our findings are strictly restricted to the Turing model of computation. 
Moreover, we assume that the input representation is part of the algorithm(s) so that only computable numbers as inputs are feasible.
Here, a complex number $c$ is computable, if there exists a Turing machine which, given an error bound $2^{-m}, m\in\N$, outputs a rational number with distance to $c$ smaller than $2^{-m}$. Hence, an algorithm takes computable numbers $x$ and an error parameter $\varepsilon>0$ as input and needs to compute an approximation of a fixed element of the solution set corresponding to $x$ with error at most $\varepsilon$. The algorithm is successful if it can accomplish this task for any computable input and error parameter. We will refer to these algorithms as Borel-Turing algorithms. 

As described, a general algorithm in the sense of \cite{bastounis21extended, colbrook21stable} generalizes this concept and additionally captures a multitude of computational models. Hence, when restricting the input domain of a general algorithm to computable numbers, it also describes Borel-Turing algorithms since the input representation in the general algorithm setting is necessarily computable. The restriction of the input domain is a critical feature since Turing machines can not compute all complex numbers to within any desired precision, i.e., the computable complex numbers are indeed a proper subset of the complex numbers. Due to the larger input domain of general algorithms in comparison to Borel-Turing algorithms, it is not a priori clear if the previously described non-existence results of general algorithms also convey to algorithms in Borel-Turing sense. However, an analysis of the proof techniques in \cite{bastounis21extended,colbrook21stable} reveals that it is indeed the case. 
We adopt ideas from \cite{bastounis21extended,colbrook21stable} and embed them in the framework of recursive functions. Thereby, we show that finite-dimensional inverse problems are not Banach-Mazur computable and we establish a lower bound on the approximation accuracy of the solution map of inverse problems via Banach-Mazur computable functions.

Finally, we would also like to mention that the described differences in the input representation are far from trivial, even if our findings and the results in \cite{bastounis21extended,colbrook21stable} are closely related in the Turing model. In \cite{Boche2022InvProb}, it was shown that in the BSS model with exact input representations, i.e., real numbers are considered as entities, the computational barriers for solving finite-dimensional inverse problems observed in the Turing model do not appear. In contrast, the non-existence of general algorithms for solving finite-dimensional inverse problems remains valid in the BSS model with inexact input representations  \cite{bastounis21extended,colbrook21stable}. Hence, based on the input representation, different capabilities in the same computing model may arise.

\subsection{Impact}
The results in \cite{bastounis21extended,colbrook21stable, bastounis2021Classification} as well as our contributions establish a limit on the capabilities of deep learning on digital machines. Since this might cause significant concern, we would like to discuss its implications in more detail.
\begin{itemize}
    \item First, we considered the application of deep learning methods on inverse problems, and our approach and results only concern this use-case. In particular, specific solution strategies --- which are typically applied in practice --- via optimization problems such as basis pursuit \cite{Candes06Stable} and lasso \cite{Tropp06Relax, Belloni11SquareRootLasso} were analyzed. The same solution strategies were also considered in \cite{bastounis21extended, colbrook21stable}. Although this is very strong evidence that inverse problems in general encompass the described limitations, a completely new approach circumventing the limits can not be ruled out, however, it appears to be highly unlikely. Moreover, the need to investigate the capabilities and limitations of deep learning persists beyond (finite-dimensional) inverse problems. Thus, it is important to consider other application fields of deep learning from the point of view of computability, as for instance carried out for classification problems in \cite{bastounis2021Classification}.
    \item Second, the described findings do not question the impressive power of deep learning, but emphasize a necessary caution. The instability and non-robustness of today's networks that lead to unreliable behaviour is a widely accepted phenomenon \cite{Akhtar18ThreatAdvAtt, Carlini18AudioAdvEx, Moosavi16DeepFool, Finlayson19AdvAttMed}. However, it is not clear how to avoid it \cite{Madry18AdvTraining, Papernot16Distillation} and why it arises \cite{Ilyas19AdvExNotBugs, Tsipras18RobustnessOdds, gilmer2018adversarial}. Is it due to some insufficient tools and methods or is it an inherent property of deep learning? We are not able to provide a comprehensive answer to this question at this point, but it is now clear that any algorithm used for training a neural network (solving an inverse problem) on a digital machine has certain computational barriers. This may not be connected to today's instabilities, but it demonstrates that we can not expect perfectly accurate networks for every situation. Additionally, the analysis in \cite{colbrook21stable, bastounis2021Classification} provides insights under which circumstances robust and accurate deep learning may be feasible. At the very least, we need to be aware of these limitations so that the user(s) can assess the potential risk of undesired effects. 
    \item Third, neural networks can generally not be trained up to arbitrary high accuracy in an automated fashion on Turing machines --- meaning that, given a desired accuracy, a Turing machine outputs a network capable of approximating the given task with the prescribed accuracy. 
    Even more, we can specify lower bounds on the achievable precision in the Banach-Mazur setting (see \Cref{thm:nonApprox})). This achievable precision is defined in \cite{colbrook21stable} and \cite{bastounis21extended} as the ‘breakdown epsilons’, wherein similar statements are made for different models of computation.
    \item At last, we stress that our results are inherently connected to digital computations. Using different hardware platforms may lead to different answers. In particular, in deep learning, the emergence of neuromorphic hardware \cite{Mead90NeuroComp, schuman17survey}  --- a combination of digital and analog computations where elements of a computer are modeled after systems in the human brain and nervous system --- demands the considerations of analog computation models given, for instance, by the BSS machine \cite{Blum89BSS}. Intriguingly, the limitations that occur on Turing machines via the Banach-Mazur framework do not necessarily arise in the BSS model \cite{Boche2022InvProb}. Here, the input representation is crucial and explains the difference to general algorithms considered in \cite{bastounis21extended, colbrook21stable}, where the limitations persist also in analog computations modeled by BSS machines.
\end{itemize}

\subsection{Outline}
A concise introduction to the theory of computation with all necessary notions will be presented in \Cref{sec:preliminaries Computation}. This is followed by a description of finite-dimensional inverse problems and how deep learning is applied to find its solutions in \Cref{sec:invProb}. \Cref{sec:mainResult} covers the formal statement of our main results concerning the limitations of deep learning and their proof. We conclude with a short discussion of our findings and their implications in \Cref{sec:discussion}.

\section{Preliminaries from the Theory of Computation}\label{sec:preliminaries Computation}

Already in 1936, Alan Turing addressed the problem of the algorithmic computation of real numbers \cite{Turing36Entscheidung}. He found that only countably many real numbers can be computed, so that most real numbers are not computable. Several different computing models were developed by Turing himself \cite{Turing36Entscheidung}, John von Neumann \cite{vonNeumann45Architektur}, Kurt Gödel \cite{Goedel31Unentscheidbar}, and many others, to identify objects which are algorithmically computable. Intriguingly, it turned out that all these computing models are equivalent. Nowadays, the most well-known and widely applied model by computer scientists is Turing's proposal --- the so-called Turing machine. 

Turing machines are a mathematical model of what we intuitively understand as computation machines that manipulate symbols on a strip of tape according to certain given rules. They yield an abstract idealization of today’s real-world (digital) computers. Any algorithm that can be executed by a real-world computer can, in theory, be simulated by a Turing machine, and vice versa. In contrast to real-world computers, Turing machines are not subject to any restrictions regarding energy consumption, computation time, or memory size. In addition, the computation of a Turing machine is assumed to be executed free of any errors.

Since its introduction, several non-computable operations where discovered, i.e., operations that can not be effectively approximated on a Turing machine and thereby not on a digital computer. Let us exemplarily mention that there exist continuously differentiable and computable functions for which the first derivative is not computable \cite{Myhill71Noncomp}.

Next, we present the relevant notions from computability theory based on Turing machines that will enable us to investigate the questions raised in the introduction. A comprehensive formal introduction on the topic of computability may be found in \cite{Soare87RecursivelyES, Weihrauch00CompAnal, Pour-El17Computability, AvigadBrattka14CompAnal}. 

We first recall two notions from classical computability theory.
\begin{Definition}
    \begin{enumerate}[(1)]
        \item[] 
        \item A set $A \subset \N$ of natural numbers is said to be \textit{recursive} (or computable or decidable), if there exists an algorithm taking natural numbers as input and which correctly decides after a finite amount of time whether the input belongs to $A$ or to the complement set $A^c = \N\setminus\!A$. 
        \item A set $A \subset \N$ is called \textit{recursively enumerable}, if there exists an algorithm, which correctly identifies (after a finite amount of time) all inputs belonging to $A$, but which may not halt if the input does not belong to $A$.
    \end{enumerate}
\end{Definition}

Note that the halting problem for Turing machines ensures the existence of recursively enumerable sets that are non-recursive. Before considering the notion of computability on an uncountable domain, we recall its meaning on a discrete domain like the natural numbers. For this purpose, we introduce the set of $\mu$-\textit{recursive functions}, often simply referred to as \textit{recursive functions} \cite{Kleene36Recursive}, which constitute a special subset of the set $\bigcup_{n=0}^\infty \{f : \N^n \hookrightarrow \N \}$, where '$\hookrightarrow$' denotes a partial mapping. Recursive functions are functions which can be constructed by applying a restricted set of operations, like composition and primitive recursion, on some simple, basic functions like constant functions and the successor function. One can think of the recursive functions as a natural mathematical class of computable functions on the natural numbers.

The considered functions are discrete, i.e., we can expect exact computations on a Turing machine; that is, $f : \N^n \hookrightarrow \N$ is computable if there exists a Turing machine that accepts input $x$ exactly if $f(x)$ is defined, and, upon acceptance, leaves the string $f(x)$ on its tape. Detailed knowledge about recursive functions is not required for our purpose. It suffices to know that Turing-computable functions and recursive functions constitute the same class of function as  shown by Turing himself \cite{Turing37Equivalence}. 
\begin{Lemma}
    A function $f : \N^n \hookrightarrow \N$ is a recursive function if and only if it is computable by a Turing machine. 
\end{Lemma}

\subsection{Computable Real Numbers}
Our next aim is to define what it means for a real number to be computable. Intuitively, a computable real number is one which can be effectively approximated to any desired degree of precision by a Turing machine given in advance. There are only countably many Turing machines so that only countably many computable real numbers exist. Any rational number is automatically computable and we can characterize computable sequences of rational numbers. 

\begin{Definition}\label{def:CompSeqRat}
    A sequence $(r_k)_{k \in \N}$ of \textit{rational numbers} is \textit{computable}, if there exist three recursive functions $a, b, s : \N \to \N$ such that $b(k) \neq 0$ and
    \begin{equation*} 
        r_k = (-1)^{s(k)} \frac{a(k)}{b(k)} \qquad \text{ for all } k \in \N. 
    \end{equation*}
\end{Definition}

\begin{Remark}
    The definition can be straightforwardly extended to multi-indexed sequences. For a fixed $n \in \N$, a sequence $(r_{k_1,\dots,k_n})_{k_1,\dots,k_n \in \N} \subset \Q$ is then called \textit{computable}, if there exist three recursive functions $a, b, s : \N^n \to \N$ such that 
    \begin{equation*}
        r_{k_1,\dots,k_n} = (-1)^{s(k_1,\dots,k_n)} \frac{a(k_1,\dots,k_n)}{b(k_1,\dots,k_n)} \qquad \text{ for all } k_1,\dots,k_n \in \N. 
    \end{equation*}
\end{Remark}

We can apply computable sequences of rational numbers to establish computability of a real number.

\begin{Definition}
\begin{enumerate}[(1)]
    \item[]
    \item A sequence $(r_k)_{k \in \N}$ of rational numbers \textit{converges effectively} to a real number $x$, if there exists a recursive function $e: \N \to \N$ such that
    \begin{equation*}
        \abs{r_k - x} \leq 2^{-N} 
    \end{equation*}
    holds true for all $N \in \N$ and all $k \geq e(N)$. 
    \item A \textit{real number} $x$ is \textit{computable}, if there exists a computable sequence $(r_k)_{k \in \N}$ of rationals which converges effectively to $x$. We refer to the sequence $(r_k)_{k \in \N}$ as a \textit{representation} for $x$ and denote the set of computable real numbers by $\R_c$.
\end{enumerate}
\end{Definition}

Next, we introduce some properties of computable numbers as well as its extension to multi-dimensional objects.

\begin{Remark}
    \begin{enumerate}[(1)]
        \item The set $\R_c$ is a subfield of $\R$ with countably many elements. Since all rational numbers are computable, $\R_c$ is dense in $\R$. 
        \item A number $x \in \R$ is computable if and only if there exists a computable sequence of rational numbers $(r_k)_{k \in \N}$ such that
        \begin{equation*} 
            \abs{r_k - x} \leq 2^{-k} 
        \end{equation*}
        for all $k\in\N$.
    \end{enumerate}
\end{Remark}
\begin{Definition}
A \textit{vector} $v \in \R^n$ is \textit{computable} if each of its components is computable. Similarly, a \textit{complex number} is \textit{computable} if its real and imaginary parts are computable. We denote the set of \textit{computable complex numbers} by $\C_c = \{x+iy\in\C : x,y\in\R_c\}$. 
\end{Definition}

An important operation is the comparison of two numbers. Unfortunately, it is not straightforward to implement it algorithmically.

\begin{Remark}
    There is no algorithm that decides whether two computable real numbers are equal. The inequality relation between computable reals is recursively enumerable, i.e., given $x,y \in \R_c$, one can effectively decide if $x>y$ or $x<y$, provided that $x\neq y$. However, for rational numbers comparisons can be (effectively) decided \cite{Pour-El17Computability}.
\end{Remark}

\subsection{Computable Sequences of Real Numbers}
Our next goal is to identify computable sequences of real numbers. Surprisingly, such a sequence might not be computable, even though each of its individual elements is. 

\begin{Definition}\label{def:SeqRealNum}
    Let $(x_n)_{n \in \N}, (x_{n,k})_{n,k \in \N}$ be a sequence, respectively a double-indexed sequence, of real numbers such that
    \begin{equation*}
        x_{n,k} \to x_n \quad  \text{ for } k \to \infty \text{ and each } n \in\N.
    \end{equation*}
    We say that $(x_{n,k})_{n,k \in \N}$ \textit{converges} to $(x_n)_{n \in \N}$ \textit{effectively in k and n}, if there exists a recursive function $e: \N \times \N \to \N$ such that
    \begin{equation*} 
        \abs{x_{n,k} - x_n} \leq 2^{-N} 
    \end{equation*}
    holds true for all $n,N \in \N$ and all $k \geq e(n,N)$. A \textit{sequence of real numbers} $(x_n)_{n\in \N}$ is \textit{computable (as a sequence)} if there exists a computable double-indexed sequence of rationals $(r_{n,k})_{n,k \in \N}$ such that $r_{n,k}$ converges to $x_n$ effectively in $k$ and $n$.
\end{Definition}

\begin{Remark}
    In the previous definition one may assume without loss of generality that the recursive function $e: \N \times \N \to \N$ is an increasing function of both variables.
\end{Remark}

Subsequently, we collect some properties of computable sequences and again describe their extension to multi-dimensional objects.

\begin{Remark}
    \begin{enumerate}[(1)]
        \item A sequence $(x_n)_{n\in \N} \subset \R$ is  \textit{computable} if and only if there exists a computable double-indexed sequence $\{r_{n,k}\}_{n,k \in \N} \subset \Q$ such that
        \begin{equation*} 
            \abs{r_{n,k} - x_n} \leq 2^{-k} \quad \text{ for all } k, n \in\N.
        \end{equation*}
        \item One immediately observes that each finite sequence of computable real or complex numbers is computable. 
    \end{enumerate}
\end{Remark}

\begin{Remark}
    We can again straightforwardly extend the computability notion to multi-indexed sequences. For a fixed $n\in\N$, we call a sequence $(x_{k_1,\dots,k_n})_{k_1,\dots,k_n \in \N} \subset \R$ \textit{computable} if and only if there exists exists a computable sequence $\{r_{k_1,\dots,k_{n+1}}\}_{k_1,\dots,k_{n+1} \in \N} \subset \Q$ of rational numbers such that
    \begin{equation*} 
        \abs{r_{k_1,\dots,k_{n+1}} - x_{k_1,\dots,k_n}} \leq 2^{-k_{n+1}} \quad \text{ for all } k_1,\dots,k_{n+1} \in\N.
    \end{equation*}
\end{Remark}

\begin{Definition}
    A \textit{sequence of vectors} $(v_k)_{k\in\N} \subset \R^n$ is \textit{computable} if each of its components is a computable sequence. A \textit{sequence of complex numbers} is \textit{computable} if its real and imaginary parts are computable sequences.
\end{Definition}

An interesting question is under which circumstances the limit of a (computable) sequence is computable. The following example illustrates that simple convergence of a sequence consisting of computable elements is not sufficient to guarantee computability of the limit. Consider a non-computable number $u \in \R$. Due to the fact that $\R_c$ is dense in $\R$, there has to exist a sequence of computable numbers $(x_n)_{n\in\N}$ such that $\abs{u - x_m} \leq 2^{-m}$ for all $m\in\N$. One might think that this contradicts the non-computability of $u$. Although the converging sequence $(x_n)_{n\in\N}$ exists, it is not computable as a sequence despite all individual elements being computable. In other words, the sequence can not be obtained algorithmically.

Under slightly stronger assumptions we obtain a highly useful result concerning the closure of multi-indexed computable sequences (cp. \cite{Pour-El17Computability}).

\begin{Theorem}[Closure under effective convergence]\label{thm:effconv}
Let $(x_{n,k})_{n,k \in \N}\subset \R$ be a computable double-indexed sequence of real numbers, which converges to a sequence $(x_n)_{n \in \N}$ as $k \to \infty$  effectively in $k$ and $n$. Then the sequence $(x_n)_{n \in \N}$ is computable.
\end{Theorem}

\subsection{Computable Functions}
Often the computability of practical problems can be reduced to the question whether a certain function is computable by formulating the problems as an input-output relation. Computability of functions is a well-studied property, and there exist various computability notions. We would like to mention Borel- and Markov computable functions or computable continuous functions. In essence, the concept of computability defines the expected input description and the admissible computation steps for solving a given task. For a comprehensive review of different computabillity notions we refer to \cite{AvigadBrattka14CompAnal}. One notion, we will employ in our paper, is the following one.

\begin{Definition}
    A function $f: I \to \R^n_c$, $I \subset \R^m_c$, is called \textit{Borel-Turing computable}, if there exists an algorithm (or Turing machine) that transforms each given computable representation of a computable vector $x \in I$ into a representation for $f(x)$.
\end{Definition}

\begin{Remark}
\begin{enumerate}[(1)]
    \item 
    In general, the representation $(r_k)_{k \in \N}$ of a computable vector $x$ is not unique. Hence, the representation of a Borel-Turing computable function $f$ at $f(x)$ depends on the representation $(r_k)_{k \in \N}$ of $x$ given as input to the algorithm. It is important to be aware of this subtlety, since, for instance, the running time of the algorithm may depend on the input representation and vary accordingly. 
    \item 
    We note that Turing’s original definition of computability conforms to the definition of Borel-Turing computability. From a practical point of view, this can be seen as a minimal requirement for an algorithmic computation of the input-output relation of a problem on perfect digital hardware. The typical approach to solve this task consists of establishing an algorithm that takes the input(s) and computes the output value(s) with a precision depending on the input precision. In particular, the algorithm can determine a sufficient input precision so that the computation terminates with an output within the prescribed worst-case error bound. The whole algorithmic computation comprises therefore the following steps. First, the sought output precision as well as the representation of the input is provided to the algorithm. Then, the algorithm determines an adequate description of the input, which guarantees upon termination of the computation that the presented output satisfies the prescribed precision. Here, one can think of the input representation as a Turing machine, which can be queried with a precision parameter, that provides an approximation of the ``exact'' input.     
\end{enumerate}
\end{Remark}

On Turing machines, the weakest form of computability on the computable numbers is Banach-Mazur computability. If a single-valued function is not Banach–Mazur computable, then it is not computable on perfect digital hardware with respect to any other reasonable notion of computability including Borel-Turing computability. Note that this statement is true only for single-valued functions. For multi-valued functions different notions of algorithmic computation may be introduced which qualify neither as stronger or weaker than Banach-Mazur computability (see \Cref{def:AlgSolv}) and \Cref{rm:genAlg} for further details).

\begin{Definition}[Banach-Mazur computability]
    A function $f: I \to \R^n_c$, $I \subset \R^m_c$, is said to be \textit{Banach-Mazur computable}, if $f$ maps computable sequences $(t_n)_{n\in\N}\subset I$ onto computable sequences $(f(t_n))_{n\in\N}\subset \R^n_c$. 
\end{Definition}

\begin{Remark}
    We can extend Borel-Turing and Banach-Mazur computability in a straightforward way to the complex domain. A function $f: I \to \C^n_c$, $I \subset \C^m_c$ is Borel-Turing, respectively Banach-Mazur, computable if its real and imaginary part are Borel-Turing, respectively Banach-Mazur, computable. Hence, the problem of computability reduces to the real-valued case.
\end{Remark}

Many elementary functions can be identified as computable. Subsequently, we present some examples of computable elementary functions to provide some intuition for the reader.

\begin{Remark}[Elementary functions]\label{rm:elementaryfct}
Let $(x_n)_{n\in \N}$ and $(y_n)_{n\in \N}$ be computable sequences of real numbers. Then the following sequences are computable in all reasonable models --- including in the sense of Banach-Mazur:
\begin{itemize}
    \item $(x_n \pm y_n)_{n\in \N}$, $(x_n y_n)_{n\in \N}$, $(\frac{x_n}{y_n})_{n\in \N}$ (if $y_n \neq 0$ for all $n$)
    \item $(\max{\{x_n,y_n\}})_{n\in \N}, (\min{\{x_n,y_n\}})_{n\in \N}$, $(\abs{x_n})_{n\in \N}$
    \item $(\exp{x_n})_{n\in \N}, (\log{x_n})_{n\in \N}$, $(\sqrt{x_n})_{n\in \N}$
\end{itemize}
\end{Remark}

Another important property is the relation of computable functions and continuity. Often, one can find statements like "Every computable real number function $f : \R^n \to \R$ is continuous" (see \cite{Brattka2016CompAnalHistory}). Hence, in order to show non-computability of a function it suffices to prove that the function is discontinuous. This is certainly not wrong, but the details matter as the following examples shows.
\begin{Example}
    Consider the function $f: [0,1] \to \{0,1\}$ given by
    \begin{equation*}
    f(x) = 
    \begin{cases} 
        1 &: \text{ if } x \text{ is a computable real number},\\[1.5pt] 
        0 &: \text{ else}. 
    \end{cases}    
    \end{equation*} 
    In this situation, $f$ is not continuous but Banach-Mazur computable, since any computable input sequence is necessarily mapped to $1$. Even more, the simplest Turing machine, which ignores the input and always outputs simply $1$, is able to correctly describe the input-output relation of $f$. The apparent contradiction to the aforementioned continuity-statement resolves when taking the input domains into account. While $f$ is well-defined for any input in $[0,1]$, Banach-Mazur computability implicitly expects only computable inputs by definition. Thus, the key is that the discontinuity is computable or can be effectively approximated --- which in case of $f$ does not hold.  
\end{Example}

\begin{Remark}
    In order to circumvent the intricacies related to computable numbers and functions, a restriction to rational inputs might appear as a reasonable approach. However, for most practical problems with continuous input quantities such a simplification is inadequate for at least two reasons:
    \begin{enumerate}
        \item Central values that commonly appear like $\pi$, $e$, etc., are not rational numbers.
        \item Even simple operations like the square root function $\sqrt{\cdot}$ do not map onto the rational numbers in general, e.g., $\sqrt{2} \notin \Q$.
    \end{enumerate}
    Thus, internal quantities of a problem are in most circumstances not rational, i.e., the computations for rational inputs nevertheless require computable real numbers.
\end{Remark}

\begin{Remark}
    Apart from the described practical shortcomings there also arise theoretical obstacles from the restriction to rational inputs. A prime example is given in \cite{Boche2020Fekete}: A computable sequence $(x_n)_{n\in\N}$ of computable real numbers, which additionally satisfy $x_n \in \Q$ for all $n\in\N$, is constructed. However, $(x_n)_{n\in\N}$ does not constitute a computable sequence of rational numbers according to \Cref{def:CompSeqRat}. Hence, identifying a sequence of rational numbers as computable according to \Cref{def:CompSeqRat} is a strictly stronger statement than identifying the same sequence of rational numbers as computable according to \Cref{def:SeqRealNum}.    
\end{Remark}

In certain applications the notion of a function may be too restrictive to model a problem. For example, consider an optimization problem, where the solution to a given input may not be unique. Therefore, the assignment of the input to the output is not unique. However, we can circumvent this obstacle and work with set-valued functions. Although a multi-valued function is in general not a function in its strictest sense, the notion is still typically formalized mathematically as a function.

\begin{Definition}[Multi-valued function]
    A \textit{multi-valued function} $f$ from a set $A$ to a set $B$ is a function $f: A \to \mathcal{P}(B)$, where $\mathcal{P}(B)$ denotes the power set of $B$ (such that $f(a)$ is non-empty for every $a\in A$ or is implicitly assumed as a partial function). We will denote a multi-valued function from $A$ to $B$ by $A\rightrightarrows B$ .
\end{Definition}

In most circumstances it is not necessary to compute the entire solution set of a problem described by a multi-valued map; it suffices to obtain exactly one feasible solution. If the restricted task of computing only one solution can not be carried out, then certainly the entire solution set can not be computed. Therefore, our goal is to study computability of a single-valued restriction of the multi-valued map.

To formalize this concept, note that for a multi-valued function $f:\mathcal{X} \rightrightarrows \mathcal{Y}$ there exists for each input $x \in \dom(f)$ at least one output $y_x^\ast \in f(x) \subset \mathcal{P}(\mathcal{Y})$. A single-valued restriction of $f$ can then be defined as the function
\begin{align*}
    \tilde{f}: \mathcal{X} &\to \mathcal{Y} \\
    x &\mapsto y_x^\ast.    
\end{align*}
We denote by $\mathcal{M}_f$ the set of all single-valued functions associated with the multi-valued function $f$. Hence, the set $\mathcal{M}_f$ encompasses all (single-valued) functions $\tilde{f}$ that are formed by restricting the multi-valued map $f$ to a single value for each input. If there exists at least one function $\hat{f} \in \mathcal{M}_f$ such that $\hat{f}$ is computable, then we can algorithmically solve the problem proposed by $f$.

\begin{Definition}\label{def:AlgSolv}
    A problem with an input-output relation described by a multi-valued function $f:\mathcal{X} \rightrightarrows \mathcal{Y}$ is \textit{algorithmically solvable on a Turing machine} in Borel-Turing or Banach-Mazur sense, if there exists a function $\tilde{f} \in \mathcal{M}_f$ that is Borel-Turing or Banach-Mazur computable, respectively. 
\end{Definition}
In particular, it is not relevant which of the (possibly infinitely many) functions in $\mathcal{M}_f$ is computable, since any of those is an appropriate solution. Even more, it may be the case that most functions in $\mathcal{M}_f$ are non-computable. As long as we succeed to show computability for just one of them, we consider the task algorithmically solvable.
\begin{Remark}\label{rm:genAlg} 
    A different notion of a successful algorithmic computation can be established via the distance to the solution set. This approach is pursued in \cite{colbrook21stable,bastounis21extended} and the (non-)existence of algorithms is studied with respect to an error parameter $\varepsilon$ describing the admissible distance of an algorithmic solution to the solution set. The aforementioned notion of the  breakdown epsilon for a given problem refers to the non-existence of algorithms computing a solution within the prescribed error bound.  More formally, an algorithm $\mathcal{A}$ is given an error parameter $\varepsilon>0$ and access to a representation $(r^x_k)_{k \in \N}$ of an input $x \in \mathcal{X}$ and is required to compute a solution $\mathcal{A}(\varepsilon,(r^x_k)_{k \in \N})$ such that
    \begin{equation}\label{eq:genAlgDist}
        \text{dist}\big(\mathcal{A}(\varepsilon,(r^x_k)_{k \in \N}), f(x)\big) \coloneqq \inf_{y \in f(x)} d (\mathcal{A}(\varepsilon,(r^x_k)_{k \in \N}),y) \leq \varepsilon,
    \end{equation}
    where $d$ is a metric on $\mathcal{Y}$. Hence, $\mathcal{A}(\varepsilon,(r^x_k)_{k \in \N})$ does not approximate a fixed element of the solution set $f(x)$ but only the distance to the whole solution set is controlled as $\varepsilon$ and\,\slash\,or the representation varies. In particular, the closest element in the solution set may depend on the representation and the magnitude of $\varepsilon$. Therefore, the existence of these types of algorithms can not be assessed via the notion of Banach-Mazur computability since the algorithms can not be described by a set of single-valued functions mapping from $\mathcal{X}$ to $\mathcal{Y}$. Thus, a problem characterized by $f$ may not be algorithmically solvable in Banach-Mazur sense, but an algorithm as described in \eqref{eq:genAlgDist} may still exist. Note that this case only arises if $f$ is indeed a multi-valued and not a single-valued function.  
\end{Remark}

\section{Inverse Problems}\label{sec:invProb}

\subsection{Inverse Problems in Imaging}\label{subsec:invProb}
We consider the following finite-dimensional, underdetermined inverse problem:
\begin{equation}\label{eq:problem}
    \text{Given noisy measurements }  y = Ax + e \in \C^m \text{ of } x \in \C^N, \text{ recover } x,
\end{equation}
where $A \in \C^{m \times N}, m< N$, is the \textit{sampling operator} (or measurement matrix), $e \in \C^m$ is a noise vector, $y \in \C^m$ is the \textit{vector of measurements} and $x \in \C^N$ is the object to recover (typically a vectorized discrete image). Classical examples from medical imaging are MRI, where $A$ encodes the Fourier transform, and CT, where $A$ encodes the Radon transform. In practice, the underdetermined setting $m < N$ is common, the reason being that in many applications the number of measurements is severely limited due to time, cost, power or other constraints. Given a specific inverse problem, the objective is to reconstruct (an approximation of) the original data $x$ from knowledge of the measurements $y$, the matrix $A$, and --- ideally --- the noise level $\|e\|$. 

A general solution strategy for inverse problems is to rewrite the model \eqref{eq:problem} in a mathematically more tractable form since \eqref{eq:problem} is in general ill-posed. Aiming to account for uncertainties of the measurements, a relaxed formulation is considered, which has a considerably simpler solution map than the original description \eqref{eq:problem}. Typically, the goal is to express \eqref{eq:problem} as an optimization problem given a sampling operator $A \in \C^{m\times N}$ and a vector of measurements $y \in \C^m$. There exist various formulations of this optimization problem, a straightforward one is given by the least-squares problem
\begin{equation}\label{eq:lsq}
    \argmin_{x\in C^N} \norm[\ell_2]{Ax-y}. \tag{ls}
\end{equation}
A minimizer of the least-squares problem can be straightforwardly obtained via the pseudoinverse of $A$. Unfortunately, the capabilities to compute a minimizer of \eqref{eq:lsq} algorithmically on Turing machines via the pseudoinverse is limited \cite{Boche2022PseudoInverse}. Additionally, the solution of \eqref{eq:lsq} is generally not unique and the individual solutions tend to display different qualitative properties, i.e., not all minimizers of \eqref{eq:lsq} are of the same value when reconsidering the original problem \eqref{eq:problem}. Thus, additional information about the image domain has to be incorporated in the solver, which is referred to as "regularization". A very successful conceptual approach in the last years are sparse regularization techniques \cite{Daubechies04SparseReg}. These methods exploit the inherent sparsity of data such as natural images in fixed transform domains (e.g. wavelets, shearlets, or discrete gradient) and reach state-of-the-art results for various image reconstruction tasks. A special case of this is the highly successful field of compressed sensing \cite{Candes06RobUnc, Donoho06CompSens, Candes06UnivEncStrat, Candes05DecLP}, which revolutionized the efficient acquisition of data. Recently, deep learning techniques, where the reconstruction algorithm can learn a priori information in a much more detailed manner from a training database, have provided a new paradigm with similar or even better performances for solving inverse problems in imaging science \cite{Schlemper18CNNvsCS}.

Typically, solution strategies consider the relaxed formulation via (quadratically constrained) \emph{basis pursuit} \cite{Candes06Stable}
\begin{equation}\label{eq:sparseprob}
    \argmin_{x \in \C^N} \norm[\ell_1]{x} \text{ such that } \norm[\ell_2]{Ax -y} \leq \varepsilon, \tag{bp}
\end{equation}
where the magnitude of $\varepsilon>0$ controls the relaxation, and unconstrained (square-root) \emph{lasso} \cite{Tropp06Relax, Belloni11SquareRootLasso}
\begin{equation}\label{eq:lasso}
    \argmin_{x \in \C^N} \lambda \norm[\ell_1]{x} + \norm[\ell_2]{Ax -y}, \tag{la}
\end{equation}
where $\lambda > 0$ replaces the role of the parameter $\varepsilon$. However, each formulation describes a similar problem to solve. Proving non-computability results for today's digital hardware on the simpler, relaxed problems \eqref{eq:sparseprob} and \eqref{eq:lasso} is a striking computational barrier for practical applications.

\subsection{Deep Learning for Inverse Problems}

In this section, we give a short introduction to deep learning with a particular focus on solving inverse problems \cite{Jin17DCNNInvProb, Adler17DNNInvProb, Ongie20DLInvProb, Schlemper18CNNvsCS, Mousavi15DLvCSSigRec}. For a comprehensive depiction of deep learning theory we refer to \cite{Goodfellow16DL} and \cite{Berner2021modernMathDL}. 

\subsubsection{Prediction Task}

Before we address deep learning theory we shortly dive into classical learning theory connected to prediction tasks. Informally, in a prediction task one is given data in a measurable space $\mathcal{Z}\coloneqq \mathcal{X} \times \mathcal{Y} \subset \R^n \times \R^m$ and a loss function $\mathcal{L}: \mathcal{M}(\mathcal{X},\mathcal{Y}) \times \mathcal{Z} \to [0,\infty]$, where $\mathcal{M}(\mathcal{X},\mathcal{Y})$ denotes the space of measurable functions from $\mathcal{X}$ to $\mathcal{Y}$. The goal is to choose a hypothesis set $\mathcal{F}\subset\mathcal{M}(\mathcal{X},\mathcal{Y})$ and derive a learning algorithm, i.e., a mapping
\begin{equation*}
    \mathcal{A} : \bigcup_{\ell\in\N} \mathcal{Z}^{\ell} \to \mathcal{F},
\end{equation*}
that uses samples $\tau=((x^{(i)}, y^{(i)}))_{i=1}^{\ell} \in \mathcal{Z}^{\ell}$ consisting of features $x^{(i)}\in\mathcal{X}$ and corresponding labels $y^{(i)}\in\mathcal{Y}$ to find a model $f_\tau = \mathcal{A}(\tau) \in \mathcal{F}$ that performs well on the training data $\tau$ and also generalizes to unseen data $z=(x,y)\in\mathcal{Z}$. The ability to generalize is measured via the loss function $\mathcal{L}$ and the corresponding loss $\mathcal{L}(f_\tau,z)$. 

In this paper, we aim to unravel some of the limitations of any such algorithm on digital machines, while focusing on the inverse problem setting. As discussed in Subsection \ref{subsec:invProb}, a very popular choice for the hypothesis set is the set of deep neural networks with a given architecture and the next subsection will give a short introduction into this set. But we would like to stress that our results in fact do hold for any hypothesis set.

\subsubsection{Neural Networks}\label{subsubsec:NNs}

Aiming to present the core idea of neural networks, we consider the simple setting of \textit{feedforward neural networks}, which to a certain extent is also the most common one. From a functional viewpoint, an \textit{L-layer feedforward neural network} is a mapping $\Phi : \R^d \to \R^m$ of the form
\begin{equation}\label{eq:NNdef}
    \Phi(x) = T_L \rho(T_{L-1}\rho(\dots \rho(T_1 x))), \quad x\in \R^d ,
\end{equation}
where $T_{\ell} : \R^{n_{\ell- 1}} \to \R^{n_{\ell}}$, $\ell=1,\dots,L$, are affine-linear maps
\begin{equation*}
    T_{\ell} x = W_{\ell} x + b_{\ell}, \quad W_{\ell} \in \R^{n_{\ell} \times n_{\ell-1}}, b_{\ell} \in \R^{n_{\ell}},
\end{equation*}
$\rho: \R \to \R$ is a non-linear function acting component-wise on a vector. The input and output dimension of the network is represented by $n_0 = d$ and $n_L = m$, respectively. The matrices $W_{\ell}$ are called \textit{weights}, the vectors $b_{\ell}$ \textit{biases}, and the function $\rho$ \textit{activation function}. Together with the \textit{width} $n=(n_1,\dots,n_L)$ and \textit{depth} $L$ the activation function determines the \textit{architecture} of a neural network, see \Cref{img:NN}.

The \textit{class of neural networks} $\mathcal{NN}_{L,n,\rho}$ from $\R^d$ to $\R^m$ describes the possible realizations of neural networks with architecture $(L,n,\rho)$. Realization refers to the mapping associated with a neural network. Hence, $\mathcal{NN}_{L,n,\rho}$ comprises all functions of the form given in \eqref{eq:NNdef}.  
\def\layersep{3.75cm}
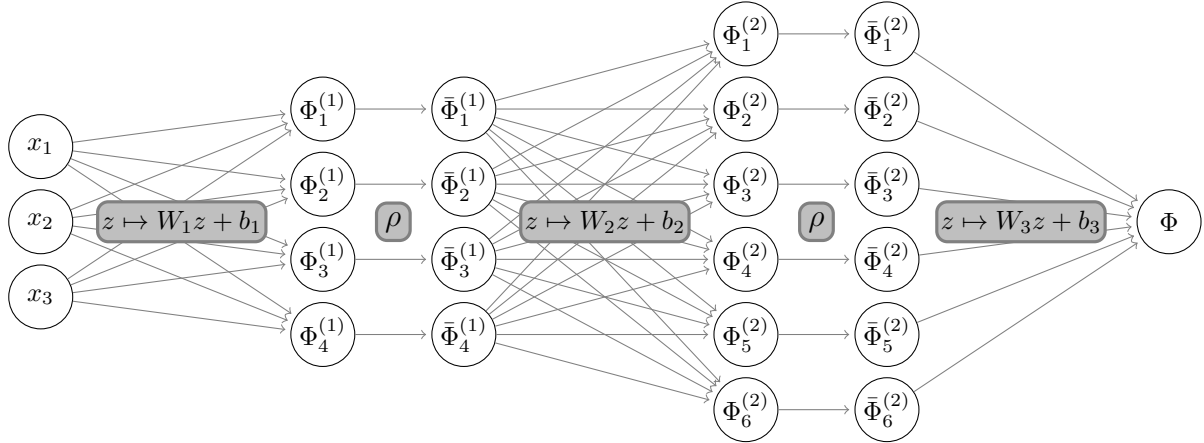
\begin{figure}[ht]
    \centering
\begin{tikzpicture}[shorten >=2pt,draw=black!50, node distance=\layersep]
    \tikzstyle{every pin edge}=[<-,shorten <=1pt]
    \tikzstyle{neuron}=[circle,draw=black,minimum size=17pt,inner sep=0pt]
    \tikzstyle{input neuron}=[neuron, minimum size=0.85cm];
    \tikzstyle{hidden neuron}=[neuron, minimum size=0.85cm];
    \tikzstyle{annot} = [rectangle, draw, very thick, rounded corners, fill=black!25,]
    \tikzstyle{annot2} = [rectangle, draw, very thick, rounded corners, fill=black!25,inner xsep= 1.5pt]

    \foreach \name / \y in {1,...,3}
        \node[input neuron] (I-\name) at (0,-\y) {$x_\y$};
    
    \foreach \name / \y in {1,...,4}
        \path[yshift=0.5cm]
            node[hidden neuron] (H-\name) at (\layersep,-\y cm) {$\Phi^{(1)}_\y$};
            
    \foreach \name / \y in {1,...,4}
        \path[yshift=0.5cm]
           node[hidden neuron] (HA-\name) at (\layersep + \layersep/2,-\y cm) {$\bar{\Phi}^{(1)}_\y$};

    \foreach \name / \y in {1,...,6}
        \path[yshift=1.5cm]
            node[hidden neuron] (H2-\name) at (2*\layersep + \layersep/2,-\y cm) {$\Phi^{(2)}_\y$};
            
    \foreach \name / \y in {1,...,6}
        \path[yshift=1.5cm]
            node[hidden neuron] (HA2-\name) at (3*\layersep,-\y cm) {$\bar{\Phi}^{(2)}_\y$};

    \foreach \name / \y in {1}
        \path[yshift=-1cm]
            node[hidden neuron] (O-\name) at (4*\layersep,-\y cm) {$\Phi$};

    \foreach \source in {1,...,3}
        \foreach \dest in {1,...,4}
            \path (I-\source) edge[->] (H-\dest);

    \foreach \source in {1,...,4}
        \foreach \dest in {1,...,6}
           \path (HA-\source) edge[->] (H2-\dest);

    \foreach \source in {1,...,6}
        \foreach \dest in {1}
            \path (HA2-\source) edge[->] (O-\dest);
        
    \foreach \source in {1,...,4}
        \path (H-\source) edge[->] (HA-\source);
        
    \foreach \source in {1,...,6}
        \path (H2-\source) edge[->] (HA2-\source);
        
    \node[annot, yshift=0.5cm](rho) at (\layersep + \layersep/4,-2.5 cm) {\large $\rho$};
    \node[annot, yshift=0.5cm](rho2) at (2*\layersep + 3*\layersep/4,-2.5 cm) {\large $\rho$};
    
    \node[annot2, yshift=0.5cm](W) at (\layersep/2,-2.5 cm) {$z \mapsto W_1 z + b_1$};
    \node[annot2, yshift=0.5cm](W2) at (2*\layersep,-2.5 cm) {$z \mapsto W_2 z + b_2$};
    \node[annot2, yshift=0.5cm](W3) at (3.475*\layersep,-2.5 cm) {$z \mapsto W_3 z + b_3$};

\end{tikzpicture}
\caption{Graph representation of a fully connected feedforward neural network $\Phi : \R^3 \to \R$ with depth $L = 3$ and widths $n_1 = 4, n_2 = 6, n_3 = 1$. The pre-activations $\Phi^{(\ell)}$ and activations $\bar{\Phi}^{(\ell)}$ of the neurons indicate the intermediate computations, whereby $\Phi^{(1)}(x) = W_1 x + b_1$ and $\Phi^{(\ell+1)}(x) = W_{\ell+1}\bar{\Phi}^{(\ell)}(x) + b_{\ell+1}$ with $\bar{\Phi}^{(\ell)}(x) = \rho(\Phi^{(\ell)}(x))$ for $\ell \in\{1,\dots,L-1\}$.} 
\label{img:NN}
\end{figure}

We now briefly recap the learning procedure of a deep neural network when aiming to solve an inverse problem \eqref{eq:problem}. Neural networks can be incorporated in the solution approach of inverse problems in various ways. Depending on the properties of a given inverse problem or the aspired application, a specific utilization of neural networks in the reconstruction process may seem the most promising. We focus on an end-to-end approach, where the goal is to directly learn a mapping from measurements to reconstructed data; see  \cite{Ongie20DLInvProb} for alternative approaches which employ deep learning at certain steps in the processing pipeline, e.g., in order to learn a regularizer. The end-to-end approach represents the most fundamental method, since it requires no further problem specific knowledge or assumptions. In this situation, our sample pairs consist of measurements $y = Ax +e$ and corresponding images $x$, i.e., the finite \textit{training set} $\tau$ is of the form
\begin{equation*}
    \tau = \{ (y^j, x^j) : y^j = Ax^j+e, j=1,\dots, R \}.
\end{equation*} 
Deep neural networks can easily be adapted to deal with complex-valued input $y \in \C^m$ and output $x \in \C^N$\!\!.
Considering vectors $y^\prime \in \R^{2m}$ and $x^\prime \in \R^{2N}$ consisting of the real and imaginary parts of $y$ and $x$, respectively, reduces the task to the real-valued case with networks mapping from $\R^{2m}$ to $\R^{2N}$\!\!.

Next the architecture, i.e., the depth $L$, the activation function $\rho$, and the width $n=(n_1,\dots,n_L)$ with input dimension $n_0=m$ and output dimension $n_L=N$, is fixed. The resulting class of neural networks $\mathcal{NN}_{L.n,\rho}$ maps from $\C^m$ to $\C^N$ and is parameterized by the weights and biases $(W_{\ell}, b_{\ell})_{\ell=1}^L$. Based on the training data $\tau$, the network $\Phi \in \mathcal{NN}_{L.n,\rho}$ is trained by solving a minimization problem of the form
\begin{equation*} 
    \Phi \in \argmin_{\tilde{\Phi}} \frac{1}{R} \sum_{\ell=1}^R \mathcal{L}(\tilde{\Phi}(y^j),x^j),
\end{equation*}
where $\mathcal{L}$ is an appropriate loss function, e.g., the $\ell_2$-loss defined as 
\begin{equation*}
    \mathcal{L}_{\ell_2}(\tilde{\Phi}(y^j),x^j) \coloneqq \|\tilde{\Phi}(y^j) - x^j\|_{\ell_2}^2.
\end{equation*}
The minimization problem is generally non-convex. Typically, there do not exist closed-form minimizers, instead numerical approaches such as gradient-type methods are required to solve the optimization task. Our concern is the following question: Can the mapping reconstructing images $x$ from measurements $y$ be effectively approximated by a neural network on digital hardware and can this network be algorithmically computed?

\begin{Remark}
    Before turning to the precise problem setting, let us already observe that the mapping realized by a neural network is computable under some very general conditions. Given that the weights, the biases, and the activation function are computable, a neural network is also computable as a composition of computable functions, namely matrix multiplication, vector addition and the application of the activation function. Typically, the learned weights and biases are computable (as the training procedure is executed on a digital computer). Additionally, the most common and universally applied activation function is the ReLU activation $\text{ReLU}(x) = \max\{x,0\}$, which is an elementary computable function. 
    Therefore, the mapping realized by a neural network can be accurately computed on digital hardware in this general setting. The key is whether an effective training algorithm, i.e., an algorithm that reliably carries out the training of a neural network to solve a given task, can exist.    
\end{Remark}

\subsubsection{Our Problem Setting}\label{subsubsec:ProbSetting}
 
We will now describe the specific problem setting as well as the precise question we aim to tackle. To solve an inverse problem (in imaging science) via a deep neural network, we first require a training set sampled from the relaxed inverse problem description. For fixed sampling operator $A$ and optimization parameter $\mu$, the training set consists of finitely many pairs $(y^j,x^j)_{j=1}^R$, where $y^j$ is a measurement and $x^j$ a corresponding solution of an optimization problem $P$, i.e.,
\begin{equation}\label{eq:tauSparse}
    \tau_{P} = \{ (y^j, x^j) : x^j \in P(A,y^j,\mu), j=1,\dots, R \},
\end{equation}
where $P(A,y^j,\mu)$ denotes the set of minimizers for the optimization problem $P$, e.g., \eqref{eq:sparseprob} (for $\mu=\varepsilon$) or \eqref{eq:lasso} (for $\mu = \lambda$). Note that the training set consists of input-output pairs of the optimization problem $P$ and not from the original inverse problem model \eqref{eq:problem}. The training process shall now yield a deep neural network $\Phi_{\tau_{P}}$, which approximates the mapping from measurements to the original data such that the recovered data satisfies the conditions posed by the applied optimization problem. Hence, we ask, for fixed sampling operator $A  \in \C^{m \times N}$ and some fixed optimization parameter $\mu > 0$:
\begin{center}
    \textit{Given measurements $\mathcal{Y} = \{y_k\}_{k=1}^{R} \subset \C^m, R < \infty$, does there exist a neural network approximating the reconstruction map, and can this network be trained by an algorithm in the Banach-Mazur sense?}
\end{center}
Such an algorithm, i.e, a training process, can by definition only exist if the reconstruction map is computable. To study this question denote by 
\begin{align*}
    \Xi_{P,A,\mu}: \C^m &\rightrightarrows \C^N \\
    y &\mapsto P(A,y,\mu)
\end{align*}
the (multi-valued) reconstruction map so that 
\begin{center}
    given a measurement $y \in \C^m$ the set of minimizers $P(A,y,\mu)$ is represented by $\Xi_{P,A,\mu}(y)$.    
\end{center}
Observe that the map $\Xi_{P,A,\mu}$ is in general set-valued, since the solution of the optimization problem does not need to be unique. However, in practical applications one is usually interested in a single-valued reconstruction map. Indeed, a neural network will yield not all eligible but exactly one recovery for each measurement. By considering single-valued restrictions $\Xi^s_{P,A,\mu} \in \mathcal{M}_{\Xi_{P,A,\mu}}$ we can examine the algorithmic solvability of a specific inverse problem determined by the sampling operator $A$, the optimization problem $P$ and parameter $\mu$. However, we are not only interested in approximating the reconstruction map of a specific inverse problem but we want to study a broader research question. Instead of fixing a specific sampling operator, we ask for any sampling operator of dimension $m\times N$: 
\begin{center}
    \textit{Given measurements $\mathcal{Y} = \{y_k\}_{k=1}^{R} \subset \C^m, R < \infty$ and the associated sampling operator $A \in \C^{m \times N}$, does there exist a neural network approximating the reconstruction map, and can this network be trained by an algorithm in the Banach-Mazur sense?}
\end{center}
The crucial aspect is that we look for an algorithm that can be applied to any inverse problem of dimension $m\times N$. Hence, the task is much more difficult than for a fixed inverse problem. The algorithm can not be adapted to properties of a specific inverse problem. Even if we can establish a suitable algorithm for a fixed inverse problem, this fact by no means implies that we can find a general algorithm applicable to any inverse problem. Such an algorithm, i.e, a training process, can by definition only exist if the reconstruction map of any inverse problem of dimension $m\times N$ is computable. More formally, denote by 
\begin{align}\label{eq:mainfunc}
    \Xi_{P,\mu}: \C^{m\times N} \times \C^m &\rightrightarrows \C^N \\  
    (A,y) &\mapsto P(A,y,\mu)\nonumber
\end{align}
the extended (multi-valued) reconstruction map so that 
\begin{center}
    given a sampling operator $A \in \C^{m \times N}$ and an associated measurement $y \in \C^m$ the set of minimizers $P(A,y,\mu)$ is represented by $\Xi_{P,\mu}(A,y)$,
\end{center}
i.e., $\Xi_{P,\mu}(A,y) = \Xi_{P,A,\mu}(y)$. Hence, the map $\Xi_{P,\mu}$ is not associated to a specific sampling operator but represents the recovery operation for any inverse problem described by a sampling operator of fixed dimension $m\times N$ and optimization task $P$ with fixed optimization parameter $\mu$). 
The algorithmic solvability related to any single-valued mapping $\Xi_{P,\mu}^s \in \mathcal{M}_{\Xi_{P,\mu}}$ is the crucial aspect for practical applications; see \Cref{img:AlgSketch}.

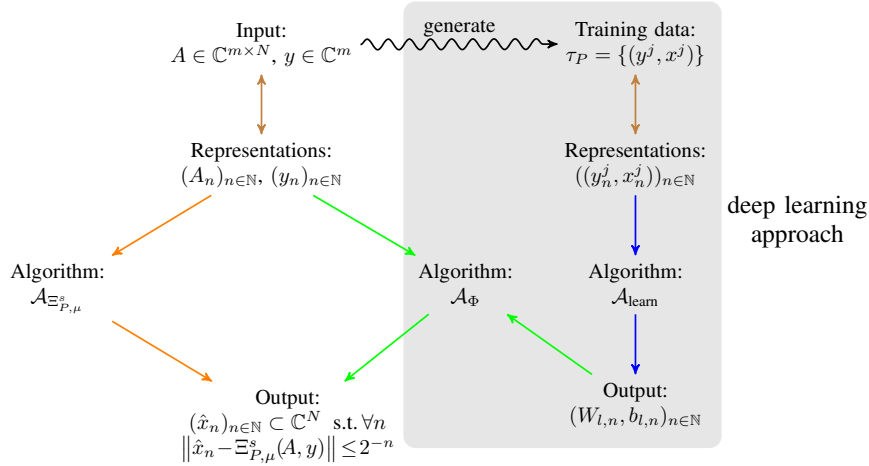
\begin{figure}[ht]
    \centering
\scalebox{0.8}{
\begin{tikzpicture}[->,>=stealth',thick] 
    \node[align=center] (input) {Input: \\ $A \in \C^{m\times N}$, $y\in \C^m$};
    \node[below=of input,align=center] (repr) {Representations:\\ $(A_n)_{n\in\N}$, $(y_n)_{n\in\N}$};
    \node[below left=of repr, align=center] (alg1) {Algorithm:\\ $\mathcal{A}_{\Xi_{P,\mu}^s}$};
    \node[below right=of alg1, align=center] (output) {Output: \\$(\hat{x}_n)_{n\in\N} \subset \C^N$ \! s.t.\! $\forall n$\\ $\norm{\hat{x}_n\! -\! \Xi_{P,\mu}^s(\! A,y)}\! \leq\! 2^{-n}$};
    
    \node[below right=of repr, align=center] (alg3) {Algorithm:\\ $\mathcal{A}_{\Phi}$};

    \node[right=of alg3,align=center] (alg2) {Algorithm:\\ $\mathcal{A}_{\text{learn}}$};
    \node[above=of alg2,align=center, label={[label distance=0.1cm, align=center, font=\large] -15:deep learning \\ approach}] (repr2) {Representations:\\ $ ((y^j_n, x^j_n))_{n\in\N} $};
    \node[above=of repr2,align=center] (data) {Training data:\\ $\tau_{P} = \{ (y^j, x^j)\}$};    
    \node[below=of alg2,align=center] (weights) {Output: \\ $(W_{l,n}, b_{l,n})_{n\in\N}$};

    \draw[<->, brown] (input) edge (repr);
    \draw[orange] (repr) edge (alg1);
    \draw[orange] (alg1) edge (output);
    
    \draw[green] (repr) edge (alg3);
    \draw[green] (alg3) edge (output);
    
    \draw[<->, brown] (data) edge (repr2);
    \draw[blue] (repr2) edge (alg2);
    \draw[blue] (alg2) edge (weights);

    \draw[green] (weights) edge (alg3);
    \draw[decorate,decoration={snake,post length=2mm}] (input) to node[above, midway] (gen) {generate}  (data);
    
    \begin{pgfonlayer}{background}
        \filldraw [line width=4mm,join=round,black!10]
        (data.north -| gen.west) rectangle (weights.south -| weights.east);
    \end{pgfonlayer}
\end{tikzpicture}
}
\caption{The goal is to construct an algorithm $\mathcal{A}_{\Xi_{P,\mu}^s}$ that takes the algorithmically generated representations $(A_n)_{n\in\N}$, $(y_n)_{n\in\N}$ (brown arrow) as input and computes a representation $(\hat{x}_n)_{n\in\N}$ of the reconstruction $\Xi_{P,\mu}^s(A,y)$, $\Xi_{P,\mu}^s \in \mathcal{M}_{\Xi_{P,\mu}}$ (orange arrows). In deep learning, a learning algorithm $\mathcal{A}_{\text{learn}}$ generates based on training data the representation $(\Phi_n)_{n\in\N}$ of a neural network parameterized by the representations $(W_{l,n}, b_{l,n})_{n\in\N}$ of the weights and biases (blue arrows) as defined in \eqref{eq:NNdef}. Subsequently, the algorithm $\mathcal{A}_{\Phi}$ employs $(\Phi_n)_{n\in\N}$ to obtain a representation $(\hat{x}_n)_{n\in\N}$ of the reconstruction from the measurement $y$ (green arrows). We already observed that $\mathcal{A}_{\Phi}$ exists under fairly general conditions in \Cref{subsubsec:NNs}, but does the algorithm $\mathcal{A}_{\text{learn}}$ exist? This can only be the case if inverse problems described by the mapping $\Xi_{P,\mu}$ are algorithmically solvable. 
}
\label{img:AlgSketch}
\end{figure}

\section{Main Result}\label{sec:mainResult}

The goal of this work is to assess whether and to what degree inverse problems described by the solution maps $\Xi_{\text{bp},\varepsilon}$ and $\Xi_{\text{la},\lambda}$ are algorithmically approximable in Banach-Mazur sense, i.e., can we algorithmically approximate optimal reconstructions? Before focusing on the Banach-Mazur setting, we want to briefly consider and formalize this problem from the Borel-Turing perspective:
Does there exist a computable sequence of Borel-Turing computable functions $(\Xi_{P,\mu}^n)_{n\in\N}$ with $\Xi_{P,\mu}^n:\C^{m\times N} \times \C^m \to \C^N$ such that, for all $(A,y) \in \C^{m\times N} \times \C^m$, we have
\begin{equation}\label{eq:appoxseq}
    \norm[\ell_2]{\Xi_{P,\mu}^s(A,y) - \Xi_{P,\mu}^n(A,y)}  < 2^{-n}    
\end{equation}
for some arbitrary function $\Xi_{P,\mu}^s \in \mathcal{M}_{\Xi_{P,\mu}}$? If sequences $(\Xi_{\text{bp},\varepsilon^\ast}^n)_{n\in\N}$ and $(\Xi_{\text{la},\lambda^\ast}^n)_{n\in\N}$ exist for certain optimization parameters $\varepsilon^\ast$ and $\lambda^\ast$, then \eqref{eq:appoxseq} implies that $\Xi^s_{\text{bp},\varepsilon^\ast}$ and $\Xi^s_{\text{la},\lambda^\ast}$ are Borel-Turing computable, respectively; for more details about this implication we refer to \cite{Boche2020SmeetsT}. Similarly, we conclude in the Banach-Mazur setting that an arbitrarily small approximation error implies that inverse problems described by the maps $\Xi_{\text{bp},\varepsilon}$ and $\Xi_{\text{la},\lambda}$ are algorithmically solvable in Banach-Mazur sense. However, our main result characterizes the non-approximability of these maps and establishes a lower bound on the algorithmically achievable accuracy from the Banach-Mazur perspective. Note that in \cite{colbrook21stable, bastounis21extended} similar lower bounds were established under the notion of 'breakdown epsilons' in the general algorithm setting. 
\begin{Theorem}\label{thm:nonApprox}
    Consider the optimization problems \eqref{eq:sparseprob} and \eqref{eq:lasso}, where $N \geq 2$ and $m < N$, for fixed parameters $\varepsilon \in (0,\sfrac{1}{4})$ and $\lambda \in (0,\sfrac{5}{4}) \cap \Q$, respectively. Let $\Xi^s_{\text{bp},\varepsilon} \in \mathcal{M}_{\Xi_{\text{bp},\varepsilon}}$ be an arbitrary (single-valued) function and let $C_{\text{bp}}>0$ be some suitably large constant. Let $\Xi:\C^{m\times N} \times \C^m \to \C^N$ be a function with
    \begin{equation}\label{eq:bound}
        \sup_{\substack{(A,y) \in \C^{m\times N} \times \C^m : \\ \norm{A} \leq K,\norm{y} \leq 1}}  \norm[\ell_2]{\Xi^s_{\text{bp},\varepsilon}(A,y) - \Xi(A,y)}  < \frac{1}{4},
    \end{equation}
    where $K \geq C_{\text{bp}}$ is arbitrary. Then $\Xi$ is not Banach-Mazur computable. 
    Similarly, for arbitrary $\Xi^s_{\text{la},\lambda} \in \mathcal{M}_{\Xi_{\text{la},\lambda}}$ and $K \geq C_{\text{la}}$, where $C_{\text{la}}>0$ is some suitably large constant,
    \begin{equation}\label{eq:nonAppLasso}
        \sup_{\substack{(A,y) \in \C^{m\times N} \times \C^m : \\ \norm{A} \leq K,\norm{y} \leq 1}}  \norm[\ell_2]{\Xi^s_{\text{la},\lambda}(A,y) - \Xi(A,y)}  < \frac{1}{8}
    \end{equation}
    implies that $\Xi:\C^{m\times N} \times \C^m \to \C^N$ is not Banach-Mazur computable. 
\end{Theorem}
\begin{Remark}
    We want to stress that the statement in \Cref{thm:nonApprox} does not depend on the unboundedness of the input domain, but holds true on a compact input set. The value of the constants $C_{\text{bp}}$ and $C_{\text{la}}$ can be explicitly computed and depend on our specific construction in the proof; they do not reflect a fundamental bound but may be improved, i.e., lowered, by further analysis.      
\end{Remark}
A similar result also holds with slightly weakened bounds when considering the distance of a function to the entire solution set.
\begin{Corollary}\label{cor:ApproxSolSet}  
    Consider the optimization problems \eqref{eq:sparseprob} and \eqref{eq:lasso}, where $N \geq 2$ and $m < N$, for fixed parameters $\varepsilon \in (0,\sfrac{1}{4})$ and $\lambda \in (0,\sfrac{5}{4}) \cap \Q$, respectively. Recall the constants $C_{\text{bp}}$ and $C_{\text{la}}$ from \Cref{thm:nonApprox}. Let $\Xi:\C^{m\times N} \times \C^m \to \C^N$ be a function with
    \begin{equation}\label{eq:boundGenAlg}
       \sup_{\substack{(A,y) \in \C^{m\times N} \times \C^m : \\ \norm{A} \leq K,\norm{y} \leq 1}} \text{dist}\big(\Xi_{\text{bp},\varepsilon}(A,y),\Xi(A,y)\big)   < \frac{1}{4} - \delta_{\text{bp}},
    \end{equation}
    where 
    \begin{equation*}
        \text{dist}\big(\Xi_{\text{bp},\varepsilon}(A,y),\Xi(A,y)\big) = \inf_{\Xi^s_{\text{bp},\varepsilon} \in \mathcal{M}_{\Xi_{\text{bp},\varepsilon}}} \norm[\ell_2]{\Xi^s_{\text{bp},\varepsilon}(A,y) - \Xi(A,y)},
    \end{equation*}
    $\delta_{\text{bp}} \in (0,\sfrac{1}{4})$ and $K \geq C_{\text{bp}}$. Then $\Xi$ is not Banach-Mazur computable. 
    Similarly, for  $\delta_{\text{la}}\in (0,\sfrac{1}{8})$ and $K \geq C_{\text{la}}$
    \begin{equation*}
        \sup_{\substack{(A,y) \in \C^{m\times N} \times \C^m : \\ \norm{A} \leq K,\norm{y} \leq 1}} \text{dist}\big(\Xi_{\text{la},\lambda}(A,y),\Xi(A,y)\big) < \frac{1}{8} - \delta_{\text{la}},
    \end{equation*}
    implies that $\Xi:\C^{m\times N} \times \C^m \to \C^N$ is not Banach-Mazur computable. 
\end{Corollary}
\begin{Remark}
    In the setting of \Cref{cor:ApproxSolSet}, Banach-Mazur non-approximability of the solution sets of \eqref{eq:sparseprob} and \eqref{eq:lasso} still relies on the evaluation of single-valued functions on the considered domain. Thus, it is not related to the algorithmic notion described in \Cref{rm:genAlg}, which can not be captured by these single-valued functions --- we refer to \Cref{rm:genAlg} for more details. 
\end{Remark}
Finally, we want to add that the limitations in \Cref{thm:nonApprox} generally do not arise due to poor conditioning of the considered problems. For specific input domains the limitations remain valid while every inverse problem in the input domain is well-conditioned. This fact was already observed in \cite{bastounis21extended, colbrook21stable} and we extend the findings to the Banach-Mazur setting. For an introduction to condition numbers we refer to \cite{Blum98ComplRealComp, Bürgisser2013condition}, the relevant notions for optimization problems are also established in \cite{colbrook21stable} and we provide a brief summary in \Cref{subsec:Cond}.
\begin{Theorem}\label{thm:Cond}
    Consider the optimization problems \eqref{eq:sparseprob} and \eqref{eq:lasso}, where $N \geq 2$ and $m < N$, for fixed parameters $\varepsilon \in (0,\sfrac{1}{4})$ and $\lambda \in (0,1) \cap \Q$, respectively. Then, there exists an input domain $\Omega \subset \C^{m\times N} \times \C^m$ containing only well-conditioned problems such that for any $\Xi^s \in \mathcal{M}_{\Xi_{\text{bp},\varepsilon}} \cup \mathcal{M}_{\Xi_{\text{la},\lambda}}$ and arbitrary Banach-Mazur computable function $\Xi:\Omega \to \C^N$  
    \begin{equation}\label{eq:OmegaCond}
       \sup_{(A,y) \in \Omega} \norm[\ell_2]{\Xi^s(A,y) - \Xi(A,y)}  \geq \frac{1}{10}
    \end{equation}
    holds. In particular, the condition numbers of matrices $A A^H$ with $(A,\cdot) \in \Omega$, the condition of the maps $\Xi_{\text{bp},\varepsilon}$ and $\Xi_{\text{la},\lambda}$ on $\Omega$, and the feasibility primal condition number on $\Omega$ are all bounded by 1.
\end{Theorem}

\subsection{Proof of Banach-Mazur Non-Approximability}

\subsubsection{Architecture of the Proof}

For the convenience of the reader we will first provide higher level insights into the overall architecture of the proof of \Cref{thm:nonApprox}. 

\begin{itemize}
    \item First, following the outline in \cite{colbrook21stable} we characterize in \Cref{lemma:generalNonApprox} general non-approximability conditions in Banach-Mazur sense for functions $\Xi_{P,\mu}$, introduced in \Cref{subsubsec:ProbSetting}, describing the solution set of an optimization problem $P$. 
    \item Informally, we show that the existence of two input sequences which converge to the same limit but whose function values are strictly separated contradicts the approximability of the function in Banach-Mazur sense under some additional requirements.
    \item Subsequently, based on a characterization of the solution sets of basis pursuit and lasso optimization summarized in \Cref{lm:SolSet}, input sequences are constructed that meet the conditions of \Cref{lemma:generalNonApprox}.
    \item Finally, analyzing the input sequences yields the degree of algorithmic non-approximability in Banach-Mazur sense of the problem posed by $\Xi_{\text{bp},\varepsilon}$ and $\Xi_{\text{la},\lambda}$ in \Cref{thm:nonApprox}.
\end{itemize}

\subsubsection{Non-Approximability Lemma}
Next, we will introduce general non-approximability conditions in Banach-Mazur sense of the mappings describing the solutions of inverse problems via optimization problems introduced in \Cref{sec:invProb}.

\begin{Lemma}\label{lemma:generalNonApprox}
    Consider for the optimization problem $P$ with optimization parameter $\mu>0$ the multi-valued mapping $\Xi_{P,\mu}$ defined in \eqref{eq:mainfunc}, i.e.,
    \begin{align*}
        \Xi_{P,\mu}: \C^{m\times N} \times \C^m &\rightrightarrows \C^N \\  
        (A,y) &\mapsto P(A,y,\mu)\nonumber
    \end{align*}
    Choose an arbitrary single-valued restriction $\Psi^s \in \mathcal{M}_{\Xi_{P,\mu}}$ and $\Omega \subseteq \dom(\Psi^s) = \{\omega =(A,y) \in \C^{m\times N} \times \C^m\} \,\vert\, P(\omega,\mu) \neq \emptyset \}$. Further, suppose that there are two computable sequences $(\omega_{n}^1)_{n\in \N}, (\omega_{n}^2)_{n\in \N} \subset \Omega$ satisfying the following conditions:
    \begin{enumerate}[(a)]
        \item There are sets $S^1, S^2 \subset \C^N$ and $\kappa > 0, \kappa \in \Q$ such that $\inf_{x_1 \in S^1, x_2\in S^2} \norm[\ell_2]{x_1 - x_2} > \kappa$ and $\Psi^s(\omega_{n}^j) \in S^j$ for $j = 1, 2$.
        \item There exists $\omega^\ast \in \Omega$ such that $\norm[\ell_2]{\omega_{n}^j - \omega^\ast} \leq 2^{-n}$ for all $n \in \N$, $j = 1, 2$. 
    \end{enumerate}
    In addition, let $\Psi: \Omega_\Psi \to \C^N$, $\Omega \subset \Omega_\Psi \subset \C^{m\times N} \times \C^m$, be an arbitrary function with
    \begin{equation*}
        \sup_{\omega \in \Omega} \norm[\ell_2]{\Psi^s(\omega) - \Psi(\omega)} < \frac{\kappa}{8}.
    \end{equation*}
    Then $\Psi$ is not Banach-Mazur computable.
\end{Lemma}
\begin{Remark}
     For $\omega \in \Omega$ we write $\omega = (\omega_A,\omega_y)$, where $\omega_A \in \C^{m\times N}$ corresponds to the matrix $A$ and $\omega_y \in \C^{m}$ corresponds to the measurement $y$. Furthermore, $\omega_{A_{k,l}}$ and $\omega_{y_k}$ denote the $(k,l)$-th element in $\omega_A$ and the $k$-th element in $\omega_y$, respectively. Moreover, we utilize the norm $\nnorm[\ell_2]{\omega} = \nnorm[\ell_2]{\omega_A} + \nnorm[\ell_2]{\omega_y}$ on $\C^{m\times N} \times \C^{m}$\!\!. Here, $\nnorm[\ell_2]{\omega_A}$ is given by the matrix norm induced by the $\ell_2$-norm.
\end{Remark}
\begin{proof}
    Without loss of generality we may assume that $\Omega, \Omega_\Psi \subset \R^{m\times N} \times \R^{m}$. The complex case $\C^{m\times N} \times \C^{m}$ is a straightforward extension of the real case. This applies, since to show Banach-Mazur computability of a complex-valued function one needs to show Banach-Mazur computability of the real-valued functions representing the real and imaginary parts. Hence, the proof reduces to the real-valued domain. 
    
    The general idea is to 'encode' a recursively enumerable but non-recursive set $B \subset \N$ in $\Omega$ in the following manner: The approximability of $\Psi^s$ on $\Omega$ by a Banach-Mazur computable function $\Psi$ would allow for an algorithm which correctly decides for arbitrary natural numbers the membership to $B$ --- contradicting its non-recursiveness. Thus, $\Psi^s$ can not be approximated in Banach-Mazur sense, i.e., $\Psi$ can not be Banach-Mazur computable. In the first part of the remaining proof we will carry out the encoding of $B$ in $\Omega$. Subsequently, in the second part we will present the contradiction arising from the Banach-Mazur computability of $\Psi$ by establishing the mentioned decision procedure. 

    \textbf{Part 1: Encoding}
    
    \noindent We begin with some observations. Note that condition (b) implies for any $k = 1,\dots,m$ and $l = 1,\dots,N$ that
    \begin{equation}\label{eq:implCondb}
        \aabs{\omega_{A_{k,l},n}^j - \omega^\ast_{A_{k,l}}} \leq 2^{-n} \quad \text{and} \quad \aabs{\omega_{y_k,n}^j - \omega^\ast_{y_k}} \leq 2^{-n} \quad \text{for all } n \in\N, j=1,2.
    \end{equation}
    Therefore, the component-wise sequences $\omega_{A_{k,l},n}^j$ and $\omega_{y_k,n}^j$ converge effectively to $\omega^\ast_{A_{k,l}}$ and $\omega^\ast_{y_k}$, respectively. Thus, $\omega^\ast_{A_{k,l}}$ and $\omega^\ast_{y_k}$ are computable real numbers. Consequently, all components of $\omega^\ast$ are computable, i.e., $\omega^\ast$ is itself computable. Moreover, the following must hold true: 
    \begin{equation*} 
        \norm[\ell_2]{\Psi^s(\omega^\ast) - \Psi^s(\omega_n^1)} > \frac{\kappa}{2} \,\,\, \forall n \in \N \quad \text{ or } \quad \norm[\ell_2]{\Psi^s(\omega^\ast) - \Psi^s(\omega_n^2)} > \frac{\kappa}{2} \,\,\, \forall n \in \N . 
    \end{equation*}
    Otherwise, there exists $n_1, n_2 \in \N$ such that
    \begin{equation*}
        \nnorm[\ell_2]{\Psi^s(\omega^\ast) - \Psi^s(\omega_{n_j}^j)} \leq \frac{\kappa}{2} \quad \text{ for } j=1,2,
    \end{equation*}
    and
    \begin{align*}
        \inf_{x_1 \in S^1, x_2\in S^2} \norm[\ell_2]{x_1 - x_2} &\leq  \inf_{x_1 \in S^1, x_2\in S^2} \norm[\ell_2]{x_1 - \Psi^s(\omega^\ast)} + \norm[\ell_2]{\Psi^s(\omega^\ast) - x_2}\\
        &\leq \norm[\ell_2]{\Psi^s(\omega_{n_1}^1) - \Psi^s(\omega^\ast)} + \norm[\ell_2]{\Psi^s(\omega_{n_2}^2) - \Psi^s(\omega^\ast)}\\
        &\leq \kappa
    \end{align*} 
    contradicts condition (a) which demands $\inf_{x_1 \in S^1, x_2\in S^2} \norm[\ell_2]{x_1 - x_2} > \kappa$. Hence, without loss of generality we can assume that 
    \begin{equation}\label{eq:kappa_cond}
        \norm[\ell_2]{\Psi^s(\omega^\ast) - \Psi^s(\omega_n^1)} > \frac{\kappa}{2} \,\,\, \text{ for all } n \in \N.
    \end{equation}
    
    Now, consider a recursively enumerable but non-recursive set $B \subset \N$. Denote by $TM_B$ a Turing machine which accepts $n\in \N$ if $n\in B$ and does not halt otherwise. We define the double-indexed sequence $(r_{n,j})_{n,j \in \N}$ based on the output of $TM_B(n)$ after $j$ steps of computation. In particular, we set
    \begin{equation*} 
        r_{n,j} = 
        \begin{cases}
            q_n &: \text{ if } TM_B \text{ accepts } n \text{ after at most } j \text{ steps},\\
            j &: \text{ otherwise},
        \end{cases}
    \end{equation*}
    where $q_n\leq j$ denotes the least number of computation steps of $TM_B$ required to accept input $n$.
    
    Next, we show that $(r_{n,j})_{n,j \in \N}$ is a computable double-indexed sequence of rational numbers. Obviously, both $q_n$ and $j$ are natural numbers, hence we need to find a recursive function $a:\N\times\N\to\N$ such that $r_{n,j} = a(n,j)$. 
    Thus, it suffices to describe a Turing machine that computes the function $a$. Informally, we can apply a Turing machine $TM_{sim}$ that simulates $TM_B$ such that $TM_{sim}(n,j)$ runs $TM_B(n)$ for at most $j$ steps. If $TM_B$ stops and accepts $n$, then $TM_{sim}$ outputs the number of required computation steps of $TM_B$, otherwise --- $TM_B(n)$ did not halt after $j$ steps --- it outputs $j$.
    
    Define the sequence $(\hat{\omega}_n)_{n\in\N} \subset  \Omega$ by
    \begin{equation}\label{eq:omegaHat}
        \hat{\omega}_n \coloneqq 
        \begin{cases}
            \omega_{q_n}^1 &: \text{ if } n \in B,\\[3pt] 
            \omega^\ast &: \text{ otherwise},
        \end{cases}
    \end{equation} 
    and note that 
    \begin{equation*}
        \omega^1_{r_{n,j}} \to \hat{\omega}_n \quad \text{ for } j \to \infty, 
    \end{equation*}
    since:
    \begin{itemize}
        \item if $n \in B$, then $TM_B$ stops after a finite number of steps, i.e. $r_{n,j} = q_n$ for $j$ large enough;
        \item if $n \notin B$, then $TM_B$ does not stop, i.e. $r_{n,j} = j$ for any $j$, and therefore $\omega^1_{r_{n,j}} = \omega^1_{j} \to \omega^\ast$ by condition (b).
    \end{itemize}
    The subsequent step is to show that $(\hat{\omega}_n)_{n\in\N}$ is a computable sequence. We consider explicitly the computability of the component-wise sequence $(\hat{\omega}_{A_{k,l},n})_{n\in\N}$ for $k,l=1$. The cases for arbitrary $k = 1,\dots,m$ and $l = 1,\dots,N$ as well as the computability of the component-wise sequences $\hat{\omega}_{y_{k},n}^1$ follow analogously. It suffices to construct a computable double-indexed sequence $(\gamma_{n,j})_{n,j\in\N}$ which converges to the sequence $(\hat{\omega}_{A_{1,1},n})_{n\in\N}$ effectively in $j$ and $n$, i.e., we need to find a recursive function $e^\ast:\N\times\N\to\N$ such that for all $n,M^\ast\in\N$ 
    \begin{equation*} 
        \aabs{\hat{\omega}_{A_{1,1},n} - \gamma_{n,j}} \leq 2^{-M^\ast} \quad  \text{ for } \quad j \geq e^\ast(n,M^\ast).
    \end{equation*}
    For $n,j \in \N$, let $\gamma_{n,j} = \omega_{A_{1,1},r_{n,j}}^1$ and observe that there exists a recursive function $e_r:\N\times\N\to\N$ with $r_{n,j} = e_r(n,j)$ since $(r_{n,j})_{n,j\in\N}$ can be computed by a Turing machine. Moreover, due to the computability of $(\omega_{A_{1,1},n}^1)_{n\in\N}$, there exists a computable double-indexed sequence $(t_{n,m})_{n,m\in\N} \subset \Q$ and a recursive function $e_t: \N \times \N \to \N$ such that
    \begin{equation*} 
        \aabs{\omega_{A_{1,1},n}^1 - t_{n,m}} \leq 2^{-M_t} 
    \end{equation*}
    holds true for all $n,M_t \in \N$ and all $m \geq e_t(n,M_t)$. By definition of a computable sequence of rationals, the elements of the sequence $(t_{n,m})_{n,m\in\N}$ can be computed via recursive functions $a_t,b_t,s_t : \N \times \N \to N$ by
    \begin{equation*}
        t_{n,m} = (-1)^{s_t(n,m)} \frac{a_t(n,m)}{b_t(n,m)} \quad \text{ for } \quad n,m \in \N.
    \end{equation*}
    Thus, the elements of the sequence $(t_{r_{n,j},m})_{n,j,m\in\N}$ can be expressed by the recursive functions $a_t,b_t,s_t,e_r$ as \begin{equation*}
        t_{r_{n,j},m} = (-1)^{s_t(e_r(n,j),m)} \frac{a_t(e_r(n,j),m)}{b_t(e_r(n,j),m)} \quad \text{ for } \quad n,j,m \in \N.
    \end{equation*}
    Observe that $\tilde{a}_t,\tilde{b}_t,\tilde{s}_t : \N^3 \to \N$ given by 
    \begin{equation*}
        \tilde{a}_t(n,j,m) \coloneqq a_t(e_r(n,j),m), \quad \tilde{b}_t(n,j,m) \coloneqq b_t(e_r(n,j),m), \quad \tilde{s}_t(n,j,m) \coloneqq s_t(e_r(n,j),m) 
    \end{equation*}
    are compositions of recursive functions and therefore again recursive. Consequently, the elements of $(t_{r_{n,j},m})_{n,j,m\in\N}$ can be computed via recursive functions $\tilde{a}_t,\tilde{b}_t,\tilde{s}_t$ as
    \begin{equation*}
        t_{r_{n,j},m} = (-1)^{\tilde{s}_t(n,j,m)} \frac{\tilde{a}_t(n,j,m)}{\tilde{b}_t(n,j,m)} \quad \text{ for } \quad n,j,m \in \N.
    \end{equation*}
    Hence, $(t_{r_{n,j},m})_{n,j,m\in\N}$ is a computable sequence of rationals. Finally, note that 
    \begin{equation*}
        \abs{\omega_{A_{1,1},r_{n,j}}^1 - t_{r_{n,j},\ell}} \leq 2^{-L}
    \end{equation*}
    holds for all $n,j,L\in\N$ and all $\ell \geq e_t(e_r(n,j),L)$. Since $e_{t,r}:\N^3\to\N$, defined by
    \begin{equation*}
        e_{t,r}(n,j,\ell) \coloneqq e_t(e_r(n,j),\ell),
    \end{equation*} 
    is recursive as a composition of recursive functions, we can conclude that $(t_{r_{n,j},m})_{n,j,m\in\N}$ converges effectively to $(\omega_{A_{1,1},r_{n,j}}^1)_{n,j\in\N}$, i.e., \Cref{thm:effconv} implies that $(\omega_{A_{1,1},r_{n,j}}^1)_{n,j\in\N}$ is a computable sequence of real numbers. 
    
    It is left to show that $(\omega_{A_{1,1},r_{n,j}}^1)_{n,j\in\N}$ converges effectively to $(\hat{\omega}_{A_{1,1},n})_{n\in\N}$. The assumptions and \eqref{eq:implCondb} imply that $(\omega_{A_{1,1},n}^1)_{n\in\N}$ is a computable sequence which converges effectively to $\omega^\ast_{A_{1,1}}$. Hence, there exists a recursive function $e_{A_{1,1}} :\N\to\N$ such that for all $M\in\N$ we obtain
    \begin{equation*}
        \aabs{\omega^\ast_{A_{1,1}} - \omega_{A_{1,1},n}^1} \leq 2^{-M} \quad  \text{ for } \quad n\geq e_{A_{1,1}}(M).    
    \end{equation*}   
    Define $e^\prime :\N\to\N$ by 
    \begin{equation*}
        e^\prime(z) \coloneqq e_{A_{1,1}}(z+1)   
    \end{equation*} 
    so that $e^\prime$ is recursive. For fixed $n\in \N\setminus\! B$ and arbitrary $M^\prime \in \N$, we have for any $j\geq e^\prime(M^\prime)$ that 
    \begin{align*}
        \aabs{\hat{\omega}_{A_{1,1},n} - \omega_{A_{1,1},r_{n,j}}^1} &= \aabs{\omega^\ast_{A_{1,1}} -\omega_{A_{1,1},r_{n,j}}^1}\\
        &= \aabs{\omega^\ast_{A_{1,1}} -\omega_{A_{1,1},j}^1}\\
        &\leq 2^{-(M^\prime+1)}.
    \end{align*}
    Similarly, for fixed $n\in\N \cap B$ it follows that   
    \begin{align*}
        \aabs{\hat{\omega}_{A_{1,1},n} - \omega_{A_{1,1},r_{n,j}}^1} &= \aabs{\omega_{A_{1,1},q_n}^1 - \omega_{A_{1,1},r_{n,j}}^1}\\
        &=
        \begin{cases} 
            \aabs{\omega_{A_{1,1},q_n}^1 - \omega_{A_{1,1},j}^1} &, \text{ if } j < q_n\\[3pt]
            \aabs{\omega_{A_{1,1},q_n}^1 - \omega_{A_{1,1},q_n}^1} &, \text{ if } j \geq q_n
        \end{cases}\\
        &\leq 
        \begin{cases} 
            \aabs{\omega_{A_{1,1},q_n}^1 - \omega^\ast_{A_{1,1}}} + \aabs{\omega^\ast_{A_{1,1}} - \omega_{A_{1,1},j}^1} &, \text{ if } j < q_n\\[3pt] 
            0 &, \text{ if } j \geq q_n
        \end{cases}\\
        &\leq 
        \begin{cases} 
            2^{-(M^\prime+1)} + 2^{-(M^\prime+1)} &, \text{ if } j < q_n\\[3pt] 
            0 &, \text{ if } j \geq q_n
        \end{cases}\\
        &\leq 2^{-M^\prime}.
    \end{align*}
    Overall, this yields for arbitrary $n\in\N$ that
    \begin{equation*} 
        \aabs{\hat{\omega}_{A_{1,1},n} - \omega_{A_{1,1},r_{n,j}}^1} \leq 2^{-M^\prime} \quad \text{ for } \quad j \geq e^\prime(M^\prime),
    \end{equation*}
    i.e., $(\omega_{A_{1,1},r_{n,j}}^1)_{j\in\N}$ converges effectively to $\hat{\omega}_{A_{1,1},n}$. Finally, setting $e^\ast(v,z) = e^\prime(z)$ guarantees that $e^\ast$ is recursive and yields for all $n,M^\ast\in\N$ that
    \begin{equation*} 
        \aabs{\hat{\omega}_{A_{1,1},n} -  \omega_{A_{1,1},r_{n,j}}^1} \leq 2^{-M^\ast} \quad \text{ for } \quad j \geq e^\ast(n,M^\ast).
    \end{equation*}
    Thereby, we have shown that $(\omega_{A_{1,1},r_{n,j}}^1)_{n,j\in\N}$ converges effectively to $(\hat{\omega}_{A_{1,1},n})_{n\in\N}$. Consequently, \Cref{thm:effconv} implies that $(\hat{\omega}_{A_{1,1},n})_{n\in\N}$ is a computable sequence. Moreover, the same holds true for any $(\hat{\omega}_{A_{k,l},n})_{n\in\N}$ and $(\hat{\omega}_{y_{k},n})_{n\in\N}$. Therefore, all component-wise sequences of $(\hat{\omega}_n)_{n\in\N}$ are computable, i.e., $(\hat{\omega}_n)_{n\in\N}$ is a computable sequence.

    \textbf{Part 2: Decision Procedure}

    \noindent Finally, we combine the previous observations to establish the claim. For the sake of contradiction assume that $\Psi$ is Banach-Mazur computable and satisfies
        \begin{equation}\label{eq:hatPsinon}
            \sup_{\omega \in \Omega} \norm[\ell_2]{\Psi^s(\omega) - \Psi(\omega)} < \frac{\kappa}{8}.
        \end{equation} 
    Then, Banach-Mazur computability of $\Psi$ implies that $\Psi(\omega^\ast)$ and $(\Psi(\hat{\omega}_n))_{n\in\N}$ are a computable vector and a computable sequence, respectively. Hence, the sequence $(s_n)_{n\in\N}$ given by
    \begin{equation*} 
        s_n \coloneqq \norm[\ell_2]{\Psi(\hat{\omega}_n) - \Psi(\omega^\ast)} \quad \text{ for } n\in\N
    \end{equation*}
    is a computable sequence since the $\ell_2$-norm consists of computable elementary functions (see \Cref{rm:elementaryfct}). 
    Observe that by construction of $(\hat{\omega}_n)_{n\in\N}$ in \eqref{eq:omegaHat} we have
    \begin{equation*}
        \begin{split}
            \Psi(\hat{\omega}_n) = 
                \begin{cases} 
                    \Psi(\omega_{q_n}^1) &: \text{ if } n \in B\\[3pt] 
                    \Psi(\omega^\ast) &: \text{ if } n \notin B 
                \end{cases}  
        \end{split}  
        \quad \text{ and } \quad
        \begin{split}
            \Psi^s(\hat{\omega}_n) = 
                \begin{cases} 
                    \Psi^s(\omega_{q_n}^1) &: \text{ if } n \in B\\[3pt] 
                    \Psi^s(\omega^\ast) &: \text{ if } n \notin B 
                \end{cases}  
        \end{split}.  
    \end{equation*}
    Due to \eqref{eq:kappa_cond} and \eqref{eq:hatPsinon} it follows that
    \begin{itemize}
        \item $s_n > \frac{\kappa}{4}$ if $n\in B$,
        since 
        \begin{align*}
            \frac{\kappa}{2} &< \norm[\ell_2]{\Psi^s(\omega^\ast) - \Psi^s(\hat{\omega}_n)} \\
            &\leq \norm[\ell_2]{\Psi^s(\omega^\ast) - \Psi(\omega^\ast)} + \norm[\ell_2]{\Psi(\omega^\ast) -\Psi(\hat{\omega}_n)} + \norm[\ell_2]{\Psi(\hat{\omega}_n) - \Psi^s(\hat{\omega}_n)} \\
            & < \frac{2\kappa}{8} + \norm[\ell_2]{\Psi(\omega^\ast) - \Psi(\hat{\omega}_n)},
        \end{align*}
        and 
        \item  $s_n < \frac{\kappa}{4}$ if $n\in \N\setminus\! B$, since
        \begin{align*}
            s_n &= \norm[\ell_2]{\Psi(\omega^\ast) - \Psi(\hat{\omega}_n)}\\
            &\leq \norm[\ell_2]{\Psi(\omega^\ast) - \Psi^s(\omega^\ast)} + \norm[\ell_2]{\Psi^s(\omega^\ast) - \Psi^s(\hat{\omega}_n)} + \norm[\ell_2]{ \Psi^s(\hat{\omega}_n) - \Psi(\hat{\omega}_n)}\\
            &< \frac{2\kappa}{8}.
        \end{align*}
    \end{itemize}
    This shows that it suffices to decide if $s_n > \sfrac{\kappa}{4}$ or $s_n < \sfrac{\kappa}{4}$ in order to determine whether $n\in B$ or $n\in \N\setminus\! B$. Since $(s_n)_{n\in\N}$ is computable, there exists a computable sequence of rationals $(p_{n,k})_{k\in\N} \subset \Q$ and a recursive function $e_{s}:\N\to\N$ such that for all $M\in\N$
    \begin{equation*} 
         \abs{s_n-p_{n,k}}\leq 2^{-M} \quad \text{ for } \quad k\geq e_{s}(M).  
    \end{equation*}
    In particular, for every $n \in \N$, using that $s_n \neq \sfrac{\kappa}{4}$, we can find $M_n \in \N$ such that
    \begin{equation}\label{eq:condStop}
        \aaabs{\frac{\kappa}{4} - p_{n,e_{s}(M_n)}} > 2^{-M_n}.
    \end{equation}
    Then, comparing the rational numbers $\sfrac{\kappa}{4}$ and $p_{n,e_{s}(M_n)}$ is sufficient to determine whether $n\in B$. Hence, we can establish the following algorithm on Turing machines that decides $B$ in a finite number of steps. On input $n\in \N$, iteratively compute $p_{n,e_{s}(M)}$ for $M=1,2,\dots$ until the condition in \eqref{eq:condStop} is met for some $M=M_n$. Subsequently, output the result of the comparison of $\sfrac{\kappa}{4}$ and $p_{n,e_{s}(M_n)}$. Due to the computability of $(s_n)_{n\in\N}$ and the fact that the final comparison only involves rational numbers, all calculations can be performed independently by a Turing machine. The computed approximations depend on the representation of $(s_n)_{n\in\N}$, however, the outcome of the comparison is independent of the actual approximation. We can conclude that there exists a Turing machine which on input $n$ decides in a finite number of steps if $n$ is in $B$. This effective method contradicts the fact that $B$ is not recursive, i.e., there can not exist a Banach-Mazur computable function $\Psi$ satisfying \eqref{eq:hatPsinon}.  
\end{proof}

\subsubsection{Construction of Input Sequences}
By certifying the conditions of \Cref{lemma:generalNonApprox} for \eqref{eq:sparseprob} and \eqref{eq:lasso} we can derive \Cref{thm:nonApprox}.
Since \Cref{lemma:generalNonApprox} relies on a characterization of the solution sets of basis pursuit and lasso optimization, we first provide an explicit description of the solution set for a specific range of optimization parameters.
\begin{Lemma}[\cite{colbrook21stable}]\label{lm:SolSet}
    Let $N \geq 2$ and consider the problem \eqref{eq:sparseprob} for
    \begin{equation*}
        A = \begin{pmatrix} a_1 & a_2 & \dots & a_N \end{pmatrix} \in \C^{1 \times N}, \quad y=1, \quad \varepsilon \in [0,1),
    \end{equation*}
    where $a_j > 0$ for $j=1,\dots,N$. Then the set of solutions is given by
    \begin{equation*}
        \sum_{j=1}^{N} \frac{t_j(1-\varepsilon)}{a_j} e_j, \quad  s.t. \quad  t_j \in [0,1], \quad \sum_{j=1}^{N} t_j = 1,\quad \text{ and } \quad t_j=0 \textit{ if } a_j < \max_k a_k,    
    \end{equation*}
    where $\{e_j\}_{j=1}^N$ is the canonical basis of $\C^N$. Furthermore, for \eqref{eq:lasso} with
    \begin{equation*}
        A = \lambda \begin{pmatrix} a_1 & a_2 & \dots & a_N \end{pmatrix} \in \C^{1 \times N}, \quad y=1, \quad \lambda>0,
    \end{equation*}
    where $a_j > 1$ for $j=1,\dots,N$, the set of solutions is given by
    \begin{equation*}
        \sum_{j=1}^{N} \frac{t_j}{\lambda a_j} e_j, \quad  s.t. \quad  t_j \in [0,1], \quad \sum_{j=1}^{N} t_j = 1,\quad \text{ and } \quad t_j=0 \textit{ if } a_j < \max_k a_k.
    \end{equation*}
\end{Lemma}
Now, we are ready to prove \Cref{thm:nonApprox}.

\begin{proof}[Proof of \Cref{thm:nonApprox}] \label{pr:Thm4.3}
    We begin with the claim concerning \eqref{eq:sparseprob} and consider the case $m=1$. The general case $m>1$ will follow by an embedding of the one-dimensional construction.
    
    Choose an arbitrary $\Xi^s_{\text{bp},\varepsilon} \in \mathcal{M}_{\Xi_{\text{bp},\varepsilon}}$. Define the sequences $(A^1_n,1)_{n\in\N}$, $(A^2_n,1)_{n\in\N} \subset  \C^{m\times N} \times \C^{m}$ where
    \begin{equation*}
        A^1_n = \begin{pmatrix} a + 2^{-n} & a & \dots & a \end{pmatrix} \quad \text{ and } \quad A^2_n = \begin{pmatrix} a & a + 2^{-n} & a & \dots & a \end{pmatrix}  
    \end{equation*}
    for $n\in\N$ and some $a \in \Q, a> 0$. The sequences converge effectively to the (computable) element $(A^\ast, 1) \in \C^{m\times N} \times \C^{m}$ with $A^\ast = \begin{pmatrix} a & \dots & a \end{pmatrix}$ since
    \begin{align}\label{eq:convOmega}
        \norm[\ell_2]{(A^j_n,1) - (A^\ast, 1)} &= \norm[\ell_2]{A^j_n - A^\ast} + \norm[\ell_2]{1 - 1}\nonumber\\
        &= a + 2^{-n} - a = 2^{-n} \quad \text{ for } j=1,2, \text{ for all } n\in \N.
    \end{align}
    Besides, observe that $(A^1_n,1)_{n\in\N}$ and $(A^2_n,1)_{n\in\N}$ are both computable sequences.  Indeed, the component-wise sequences are given either by constant sequences, which are therefore computable, or by the sequence $(a+2^{-n})_{n\in\N}$. We have
    \begin{equation*}
        a+2^{-n} = \frac{c}{d} + 2^{-n} = \frac{2^n c + d}{d2^n}
    \end{equation*}
    for some $c,d \in \N$. Consequently, there exist recursive functions $g,h:\N\to\N$, given by $g(n) = 2^n c + d$ and $h(n) = 2^n d$, such that 
    \begin{equation*}
        a+2^{-n} = \frac{g(n)}{h(n)}
    \end{equation*}
    for arbitrary $n\in\N$. Since $(a+2^{-n})_{n\in\N}$ can be expressed as the quotient of two recursive functions, it is by definition a computable sequence of rational numbers. 
    
    Finally, setting $S^j \coloneqq \{ (1-\varepsilon) s^{-1} e_j \,|\, s \in (a,a+1/2] \}$ and applying \Cref{lm:SolSet} gives 
    \begin{equation}\label{eq:PsiOmegaseq}
        \Xi^s_{\text{bp},\varepsilon}(A^j_n,1) \in \Xi_{\text{bp},\varepsilon}(A^j_n,1) = \Big\{ \frac{1-\varepsilon}{a+2^{-n}}e_j \Big\} \in S^j \quad \text{ for all } n\in\N, 
    \end{equation}
    and
    \begin{align}\label{eq:inf}
        \inf_{s_1 \in S^1, s_2 \in S^2} \norm[\ell_2]{s_1-s_2} &= \inf_{z_1,z_2 \in (a,a+1/2]} \norm[\ell_2]{(1-\varepsilon)z_1^{-1} e_1 - (1-\varepsilon)z_2^{-1} e_2}\nonumber\\
        &= (1-\varepsilon) \inf_{z_1,z_2 \in (a,a+1/2]} \sqrt{z_1^{-2} + z_2^{-2}}\nonumber \\
        &= \sqrt{2}(1-\varepsilon) (a+\tfrac{1}{2})^{-1} > 0.
    \end{align}
    Hence, the sequences $(A^j_n,1)_{n\in\N}$, $j=1,2$, satisfy all conditions in \Cref{lemma:generalNonApprox} on  
    \begin{equation*}
         \Omega= \bigcup_{j\in\{1,2\}}\bigcup_{n\in\N} \{(A^j_n,1)\} \cup \{(A^\ast,1)\} \quad \text{ with } \quad  \kappa_a = \sqrt{2}(1-\varepsilon) (a+\tfrac{1}{2})^{-1}     
    \end{equation*}
    Thus, invoking \Cref{lemma:generalNonApprox} yields that there does not exist a Banach-Mazur computable function $\Xi:\C^{m\times N} \times \C^m \to \C^N$ with  
    \begin{equation*}
        \sup_{\omega \in \Omega}  \norm[\ell_2]{\Xi^s_{\text{bp},\varepsilon}(\omega) - \Xi(\omega)} < \frac{\kappa_a}{8}.
    \end{equation*}
    To identify the optimal upper bound, we need to optimize $\kappa_a$, i.e., find the corresponding optimal sequence. For fixed $\varepsilon \in (0,\tfrac{1}{4})$ the value of $\kappa_a$ only depends on $a$, where $a$ is an arbitrary positive rational number. Therefore, we can compute
    \begin{equation*}
        \kappa_\varepsilon \coloneqq \sup_{a\in\Q, a>0} \kappa_a = \sup_{a\in\Q, a>0} \frac{\sqrt{2}(1-\varepsilon)}{a+\tfrac{1}{2}} = 2\sqrt{2}(1-\varepsilon)
    \end{equation*}
    and 
    \begin{equation*}
        \inf_{\varepsilon \in (0,\frac{1}{4})} \kappa_\varepsilon = \inf_{\varepsilon \in (0,\frac{1}{4})} 2\sqrt{2}(1-\varepsilon) = \frac{3\sqrt{2}}{2} > 2.
    \end{equation*}
    Consequently, choosing $a$ sufficiently close to 0 yields sequences which satisfy the conditions in \Cref{lemma:generalNonApprox} with $\kappa_a = 2$. Additionally, $\Omega$ is a bounded set for a fixed choice of $a$, in particular, every $(A,y)\in \Omega$ satisfies  $\norm[\ell_2]{y} = 1$ and $C_{bp}\coloneqq \sup_{(A,y) \in \Omega} \norm[\ell_2]{A} < \infty$. Thus, the claim regarding \eqref{eq:sparseprob} follows for $m=1$.

    In case of arbitrary $m>1$, observe that for 
    \begin{equation*}
        y^{(m)} = \begin{pmatrix} y & 0 & \dots & 0 \end{pmatrix} \in \C^m \quad \text{ and } \quad A^{(m)} = \begin{pmatrix} A & 0 \\ 0 & I_{m-1} \end{pmatrix} \in \C^{m\times (N+m-1)},
    \end{equation*}
    where $A \in \C^{1\times N}$, $y \in \C$ and $I_{m-1}\in \C^{(m-1)\times(m-1)}$ is the identity matrix, the constraint in \eqref{eq:sparseprob} reads as 
    \begin{align*}
        \nnormsquaared[\ell_2]{A^{(m)}x - y^{(m)}} &= \big((A^{(m)} x e_1 - y^{(m)})^T e_1\big)^2 + \sum_{i=2}^m \big((A^{(m)} x - y^{(m)})^T e_i \big)^2\\
        &= (Ax^{(N)}-y)^2 + (x^{(m-1)})^T x^{(m-1)}\\
        &\quad \text{ for any } x = \begin{pmatrix} x^{(N)} \\ x^{(m-1)} \end{pmatrix} \in \C^{N+m-1}.
    \end{align*}
    Hence, any minimum norm solution of \eqref{eq:sparseprob} requires $x^{(m-1)}$ to be zero. Thus, we can express \eqref{eq:sparseprob} for given $(A^{(m)},y^{(m)})$ equivalently by
    \begin{equation*}
        \argmin_{x^{(N)} \in \C^{N}}\,\, \nnorm[\ell_1]{x^{(N)} } \text{ such that } \nnorm[\ell_2]{A x^{(N)} - y} \leq \varepsilon.
    \end{equation*}
    Recall that by construction $A \in \C^{1\times N}$, i.e., we have reduced the problem to the one-dimensional case and are again in the previously considered setting. 

    With minor modifications we can derive the statement concerning \eqref{eq:lasso}. It suffices to consider the one-dimensional case $m=1$ again since a similar embedding as for \eqref{eq:sparseprob} can be constructed. Let $\Xi^s_{\text{la},\lambda} \in \mathcal{M}_{\Xi_{\text{la},\lambda}}$ and note that the sequences $(\lambda A^1_{n+3},1)_{n\in\N}$ and $(\lambda A^2_{n+3},1)_{n\in\N}$, where we require that $a \in \Q, a> 1$, satisfy the assumptions in \Cref{lemma:generalNonApprox}. In particular, via \Cref{lm:SolSet} we obtain 
    \begin{equation*}
         \Xi^s_{\text{la},\lambda}(A^j_{n+3},1)\in \Xi_{\text{la},\lambda}(A^j_{n+3},1) = \{ \lambda^{-1}(a+2^{-(n+3)})^{-1}e_j\}  \in S^j \quad \text{ for all } n\in\N 
    \end{equation*}
    and $S^j \coloneqq \{ \lambda^{-1} s^{-1} e_j \,|\, s \in (a,a+2^{-4}] \}$. Since
    \begin{equation*}
        \inf_{s_1 \in S^1, s_2 \in S^2} \norm[\ell_2]{s_1-s_2}= \frac{\sqrt{2}}{\lambda} \big(a+2^{-4}\big)^{-1}
    \end{equation*}    
    we have by \Cref{lemma:generalNonApprox} that there does not exist a Banach-Mazur computable function $\Xi:\C^{m\times N} \times \C^m \to \C^N$ with  
    \begin{equation*}
        \sup_{\omega \in \Omega}  \norm[\ell_2]{\Xi^s_{\text{la},\lambda}(\omega) - \Xi(\omega)} < \frac{\kappa_a}{8},
    \end{equation*}
    where 
    \begin{equation*}
         \Omega= \bigcup_{j\in\{1,2\}}\bigcup_{n\in\N} \{(\lambda A^j_{n+3},1)\} \cup \{(\lambda A^\ast,1)\} \quad \text{ with } \quad  \kappa_a = \frac{\sqrt{2}}{\lambda} \big(a+2^{-4}\big)^{-1}.
    \end{equation*}
    Finally, choosing the optimal $\kappa_a$ via 
    \begin{equation*}
        \kappa_\lambda \coloneqq \sup_{a\in\Q, a>1} \kappa_a = \sup_{a\in\Q, a>1} \frac{\sqrt{2}}{\lambda} \big(a+2^{-4}\big)^{-1} = \frac{\sqrt{2}}{\lambda(1+2^{-4})}
    \end{equation*}
    and 
    \begin{equation*}
        \inf_{\lambda \in (0,\frac{5}{4})} \kappa_\lambda = \inf_{\lambda \in (0,\frac{5}{4})} \frac{\sqrt{2}}{\lambda(1+2^{-4})} = \frac{4 \sqrt{2}}{5(1+2^{-4})} > 1
    \end{equation*}
    gives the claim.
\end{proof}

\begin{Remark}
    With minor modifications to the applied constructions, the bounds in \eqref{eq:bound} and \eqref{eq:nonAppLasso} can be strengthened, i.e., the range of the optimization parameters as well as the non-approximation bounds increased and the constants decreased. The stated quantities in \Cref{thm:nonApprox} were derived by analyzing specific sequences constructed in the proof. By further optimizing the construction stronger results can be obtained, e.g., it easily follows from \eqref{eq:PsiOmegaseq} and \eqref{eq:inf} that removing the first $k$ elements, for some $k \in \N$, from the considered sequences strengthens the quantities in the basis pursuit instance. 
\end{Remark}

\Cref{thm:nonApprox} immediately implies the statement in \Cref{cor:ApproxSolSet} concerning the non-approximability in Banach-Mazur sense of the entire solution set.
\begin{proof}[Proof of \Cref{cor:ApproxSolSet}]
    We consider the statement related to \eqref{eq:sparseprob}, the result for \eqref{eq:lasso} follows analogously. Let $\Xi^s_{\text{bp},\varepsilon} \in \mathcal{M}_{\Xi_{\text{bp},\varepsilon}}$ be a single-valued solution function which maps feasible inputs $(A,y)\in \C^{m\times N}\times \C^m$ to $x_{A,y} \in \Xi_{\text{bp},\varepsilon}$ such that $x_{A,y}$ additionally satisfies
    \begin{equation*}
        \norm[\ell_2]{x_{A,y} - \Xi(A,y)} < \frac{1}{4} -\frac{\delta_{\text{bp}}}{2} \quad \text{ on } \{(A,y) \in \C^{m\times N} \times \C^m \,|\, \\ \norm[\ell_2]{A} \leq K,\norm[\ell_2]{y} \leq 1\}
    \end{equation*}
    for $K \geq C_{\text{bp}}$. Note that \eqref{eq:boundGenAlg} guarantees the existence of $\Xi^s_{\text{bp},\varepsilon}$ and we obtain
    \begin{equation*}
        \sup_{\substack{(A,y) \in \C^{m\times N} \times \C^m : \\ \norm{A} \leq K,\norm{y} \leq 1}} \norm[\ell_2]{\Xi^s_{\text{bp},\varepsilon}(A,y) - \Xi(A,y)} =  \sup_{\substack{(A,y) \in \C^{m\times N} \times \C^m : \\ \norm{A} \leq K,\norm{y} \leq 1}} \norm[\ell_2]{x_{A,y}- \Xi(A,y)} < \frac{1}{4}.
    \end{equation*}
    Hence, \Cref{thm:nonApprox} yields that $\Xi$ is not Banach-Mazur computable.
\end{proof}

\subsection{Proof of Well-Conditioning}\label{subsec:Cond}
Next, we will give a short overview of standard notions of condition numbers related to inverse problem and optimization. For more detailed derivations and explanations we refer to \cite{Blum98ComplRealComp, Bürgisser2013condition}. We focus on the condition number of matrices and multi-valued mappings as well as the distance to infeasibility. The outline and the applied proof techniques closely follow \cite{colbrook21stable, bastounis21extended}.

For consistency reasons we apply the matrix norm induced by the $\ell_2$-norm, however, the following notions also are valid with respect to different norms. The classical condition number of an invertible matrix $A$ is given by 
\begin{equation*}
    \kappa(A) = \norm[\ell_2]{A}\nnorm[\ell_2]{A^{-1}}.    
\end{equation*}
The condition of a multi-valued mapping $\Xi : \Omega \subset \C^n \rightrightarrows \C^m$ depends on the set of relevant perturbations. Denote by
\begin{equation*}
    \text{Act}(\Omega) = \{ j \in \{1,\dots,n\} : \exists \, v,w \in \Omega \text{ with } v_j \neq w_j\}
\end{equation*}
the set of coordinates that are varying in $\Omega$ and, for $\alpha>0$ (where $\alpha=\infty$ is valid as well), by
\begin{equation*}
    \Omega_\alpha = \{ v \in \C^n : \exists\, w \in \Omega \text{ with } \norm[\ell_\infty]{v-w} \leq \alpha \text{ and } v_{\text{Act}(\Omega)^c} = w_{\text{Act}(\Omega)^c}\}
\end{equation*}
the set of $\alpha$-perturbations along the non-constant coordinates of elements in $\Omega$. Then, the condition number of the mapping $\Xi$ is 
\begin{equation*}
    \kappa(\Xi,\Omega) = \sup_{v \in \Omega} \lim_{\delta \to 0^+} \sup_{\substack{v +w \in \Omega_\alpha : \\ 0<\norm{w} \leq \delta}} \frac{\text{dist}\big(\Xi(v+w),\Xi(v)\big)}{\norm[\ell_2]{w}},      
\end{equation*}
where we assumed that $\Xi$ is also defined on $\Omega_\alpha$. Here, $\text{dist}(\cdot,\cdot)$ denotes the usual distance function on sets, i.e.,
\begin{equation*}
    \text{dist}\big(\Xi(v+w),\Xi(v)\big) = \inf_{z_1 \in \Xi(v+w), z_2 \in \Xi(v) } \norm[\ell_2]{z_1 - z_2}.
\end{equation*}
Finally, the distance to infeasibility is established for basis pursuit (for lasso is the following condition number is always zero), i.e., $\Xi=\Xi_{\text{bp},\varepsilon}$, via 
\begin{align*}
    \alpha(A,y) = \sup \{ \delta\geq 0 : \norm[\ell_2]{\hat{y}}, \nnorm[\ell_2]{\hat{A}} &\leq \delta, (A+\hat{A}, y +\hat{y}) \in \Omega_\infty \\ 
    &\implies (A+\hat{A}, y +\hat{y}) \text{ are feasible inputs to } \Xi_{\text{bp},\varepsilon}\}.
\end{align*}
Here, the input to $\Xi_{\text{bp},\varepsilon}$ is feasible, if the solution set of the corresponding optimization problem is non-empty. We define the Feasibility Primal (FP) condition number as
\begin{equation*}
    \kappa_{\text{FP}}(A,y) = \frac{\max \{\norm[\ell_2]{y},\norm[\ell_2]{A}\}}{\alpha(A,y)}.
\end{equation*}

Now, we are ready to apply the introduced notions of condition numbers to our problem setting.

\begin{proof}[Proof of \Cref{thm:Cond}]
    We will analyze the one-dimensional case $m=1$. The general case $m>1$ follows from the following embedding. We already observed in the \hyperref[pr:Thm4.3]{proof} of \Cref{thm:nonApprox} that, for 
    \begin{equation*}
        y^{(m)} = \begin{pmatrix} y & 0 & \dots & 0 \end{pmatrix} \in \C^m \quad \text{ and } \quad A^{(m)} = \begin{pmatrix} A & 0 \\ 0 & I_{m-1} \end{pmatrix} \in \C^{m\times (N+m-1)},
    \end{equation*}
    where $A \in \C^{1\times N}$, $y \in \C$ and $I_{m-1}\in \C^{(m-1)\times(m-1)}$ is the identity matrix, the optimization problems reduce to a one-dimensional task involving $A$ and $y$. By introducing an additional parameter $\beta_A \in \C$ via
    \begin{equation*}
        A^{(m)} = \begin{pmatrix} A & 0 \\ 0 & \beta_A I_{m-1} \end{pmatrix},    
    \end{equation*}
    where $\beta_A$ is chosen such that $A A^H$ is a multiple of the identity, we can guarantee that the embedding does not affect the considered condition numbers -- the matrix condition, the distance to infeasibility and the conditions of the mappings $\Xi_{\text{bp},\varepsilon}$, $\Xi_{\text{la},\lambda}$ remain unchanged. Therefore, it is indeed sufficient to construct $\Omega$ for $m=1$. Set
    \begin{equation*}
        \Omega = \{ (A,y) \in \C^{m\times N}\times \C^m: y=1, A= \begin{pmatrix} a_1 & a_2 & \dots & a_N \end{pmatrix} \text{ with } a_i \in (1, b) \} 
    \end{equation*}
    for some $ b \geq \sfrac{3}{2}$. Next, we check that the condition number bounds hold. Let $(A,y) \in \Omega$ and note that $\kappa(A A^H) = 1$. Moreover, any $(A^\prime,y) \in \C^{m\times N}\times \C^m$ is also in $\Omega_\infty$ so that $(A+A^\prime,y) \in \Omega_\infty$. If $\nnorm[\ell_2]{A^\prime} < \norm[\ell_2]{A}$, then $(A+A^\prime,y)$ is a feasible input to $\Xi_{\text{bp},\varepsilon}$, i.e., $\alpha(A,y) \geq \norm[\ell_2]{A} > 1$. Thus, we obtain
        \begin{equation*}
            \kappa_{\text{FP}}(A,y) = \frac{\max \{\norm[\ell_2]{y},\norm[\ell_2]{A}\}}{\alpha(A,y)} \leq  \max \Big\{\frac{1}{\norm[\ell_2]{A}}, 1 \Big\}  \leq 1.    
        \end{equation*}
    It is left to show that the condition of the mappings are bounded. Let $(A^\prime,1)\in \Omega$ with $0<\norm[\ell_2]{A-A^\prime}$ sufficiently small such that there exists $t^\prime \in [0,1]^N$, $\sum_j t^\prime_j = 1$ satisfying
    \begin{equation*}
        \sum_{j=1}^{N} \frac{t^\prime_j(1-\varepsilon)}{a_j} e_j \in \Xi_{\text{bp},\varepsilon}(A,1) \quad \text{ and }  \quad\sum_{j=1}^{N} \frac{t^\prime_j(1-\varepsilon)}{a^\prime_j} e_j \in \Xi_{\text{bp},\varepsilon}(A^\prime,1).
    \end{equation*}
    The existence of $t^\prime$ follows from the description of the solution set of \eqref{eq:sparseprob} in \Cref{lm:SolSet}. This implies 
    \begin{align*}
        \text{dist}\big( \Xi_{\text{bp},\varepsilon}(A,1), \Xi_{\text{bp},\varepsilon}(A^\prime,1)\big)\!\! &\leq (1-\varepsilon) \nnnorm[\ell_2]{\sum_{j=1}^{N} t^\prime_j \Big(\frac{1}{a_j} -\frac{1}{a^\prime_j}\Big)  e_j} \leq (1-\varepsilon) \nnnorm[\ell_2]{\sum_{j=1}^{N} \Big(\frac{1}{a_j} -\frac{1}{a^\prime_j}\Big)  e_j}\\
        &= (1-\varepsilon) \sqrt{\sum_{j=1}^{N} \Big(\frac{a^\prime_j - a_j}{a_j a^\prime_j}\Big)^{\!2}}\leq   (1-\varepsilon) \norm[\ell_2]{A-A^\prime}
    \end{align*}
    and consequently we obtain
    \begin{equation*}
        \lim_{\delta \to 0^+} \sup_{\substack{(\hat{A}+A,1) \in \Omega : \\ \nnorm{\hat{A}} \leq \delta}} \frac{\text{dist}\big(\Xi_{\text{bp},\varepsilon}(A,1),\Xi_{\text{bp},\varepsilon}(\hat{A}+A,1)\big)}{\nnorm[\ell_2]{\hat{A}}} \leq (1-\varepsilon) < 1
    \end{equation*}
    so that
    \begin{equation*}
        \kappa(\Xi_{\text{bp},\varepsilon},\Omega) \leq 1.
    \end{equation*}
    Similarly, using that we can rewrite $A$ as
    \begin{equation*}
        A = \lambda \begin{pmatrix} \frac{a_1}{\lambda} & \frac{a_2}{\lambda} & \dots & \frac{a_N}{\lambda} \end{pmatrix},
    \end{equation*}
    where $\sfrac{a_j}{\lambda}> 1$ by construction, we can approach the estimation of the condition number of $\Xi_{\text{la},\lambda}$. Let $(A^\prime,1) \in \Omega$ such that $0<\norm[\ell_2]{A-A^\prime}$ is sufficiently small and there exists $t^\prime \in [0,1]^N$, $\sum_j t^\prime_j = 1$ satisfying
        \begin{equation*}
            \sum_{j=1}^{N} \frac{t^\prime_j}{\lambda \frac{a_j}{\lambda}} e_j \in \Xi_{\text{la},\lambda}(A,1) \quad \text{ and }  \quad\sum_{j=1}^{N} \frac{t^\prime_j}{\lambda \frac{a^\prime_j}{\lambda}} e_j \in \Xi_{\text{la},\lambda}(A^\prime,1).
        \end{equation*}
        The existence of $t^\prime$ follows again from the description of the solution set of \eqref{eq:lasso} in \Cref{lm:SolSet}. Therefore, we get
        \begin{align*}
            \text{dist}\big( \Xi_{\text{la},\lambda}(A,1), \Xi_{\text{la},\lambda}(A^\prime,1)\big) &\leq \nnnorm[\ell_2]{\sum_{j=1}^{N} t^\prime_j \Big(\frac{1}{a_j} -\frac{1}{a^\prime_j}\Big)  e_j} \leq \nnnorm[\ell_2]{\sum_{j=1}^{N} \Big(\frac{1}{a_j} -\frac{1}{a^\prime_j}\Big)  e_j}\\
            &= \sqrt{\sum_{j=1}^{N} \Big(\frac{a^\prime_j - a_j}{a_j a^\prime_j}\Big)^{\!2}}<  \norm[\ell_2]{A-A^\prime}
        \end{align*}
        and consequently 
        \begin{equation*}
            \lim_{\delta \to 0^+} \sup_{\substack{(\hat{A}+A,1) \in \Omega : \\ \nnorm{\hat{A}} \leq \delta}} \frac{\text{dist}\big(\Xi_{\text{la},\lambda}(A,1),\Xi_{\text{la},\lambda}(\hat{A}+A,1)\big)}{\nnorm[\ell_2]{\hat{A}}} < 1
        \end{equation*}
        so that
        \begin{equation*}
            \kappa(\Xi_{\text{la},\lambda},\Omega) \leq 1.
        \end{equation*}
    Finally, \Cref{lemma:generalNonApprox} (via similar constructions as in the \hyperref[pr:Thm4.3]{proof} of \Cref{thm:nonApprox}) yields the non-approximability of $\Xi^s$ \eqref{eq:OmegaCond} on $\Omega$ and thereby the result.
\end{proof}

\section{Discussion}\label{sec:discussion}
Our main goal was to analyze limitations of deep learning on digital hardware. To do so, we focused on inverse problems since deep learning approaches offer state of the art results in this application area. Nevertheless, there exist a wide range of approaches to tackle inverse problem with various (dis-)advantages. We want to exemplarily demonstrate the effect of our findings on these approaches. A particularly interesting one is given by iterative algorithms \cite{Candes2011Nesta, Eldar2021ComprImag}, which can, in fact, also be combined with deep learning techniques \cite{LeCun2010Lista, Eldar2021AlgUnroll}.  

\subsection{Iterative Algorithms in the Banach-Mazur Model}
The general idea is the following: Given the optimization problem \eqref{eq:sparseprob} or \eqref{eq:lasso}, a sequence of reconstructions $(x^{(n)})_{n\in\N} \subset \C^N$ is computed that converges (ideally) to a minimizer $x^\ast$ of the problem. Thereby, the iterations are defined by the choice of the initialization $x^{(0)} \in \C^N$ and
\begin{equation}\label{eq:ItAlg}
    x^{(n)} = T(I,\mu,x^{(n-1)}),
\end{equation}
where $T$ is an operator which has access to the input $I=(A,y,\varepsilon)$ (or $I=(A,y,\lambda)$), a set of parameters $\mu$ (e.g., a step size) and the current reconstruction $x^{(n-1)}$. In practical applications, additionally a stopping criterion is required to abort the iterative process once the reconstruction is expected to be sufficiently close to the solution $x^\ast$. Ideally, the algorithm would take an error parameter $\gamma>0$ as additional input and halt, if for some $m\in\N$ the reconstruction error of $x^{(m)}$ is below $\gamma$. The criterion is usually based on heuristics which typically lead to acceptable results, e.g., the difference of $\lVert x^{(n)}\rVert_{\ell_1}$ and $\norm[\ell_1]{x^\ast}$. However, they do not ensure closeness of $x^{(m)}$ and $x^\ast$. Then again, one can guarantee convergence or even a specific convergence rate of $(x^{(n)})_{n\in\N}$ to $x^\ast$ in certain circumstances \cite{Daubechies04SparseReg, Bredies2008LinConv, Beck2009FISTA}. 

Since our findings are connected to the computation device and not to a specific solution technique, they also affect implementations of iterative algorithms on digital hardware. 
We immediately observe that iterative algorithms can not be an effective procedure in Banach-Mazur sense.
Otherwise, the existence of an effective algorithm solving inverse problems would contradict their algorithmic non-solvability. This may be surprising in light of the convergence guarantees of $(x^{(n)})_{n\in\N}$, but we want to point out some obstacles.
\begin{itemize}
    \item The convergence of the sequence $(x^{(n)})_{n\in\N}$ may depend on the choice of the initialization $x^{(0)}$ and the parameters $\mu$. Although the existence of suitable initial values and parameters to guarantee convergence can be proven, this does not imply that these values can be effectively obtained, i.e., can be effectively computed. Thus, it is essential to consider the initialization as an integral part of the algorithm when checking for effectiveness.
    \item Although the sequence $(x^{(n)})_{n\in\N}$ may indeed converge towards $x^\ast$, the convergence may not be effective. Hence, the reconstructions computed by the algorithm come without explicit error bounds on the difference to the sought solution $x^\ast$. Therefore, there does not exist an effective stopping criterion, i.e., a criterion that aborts the computation once a reconstruction satisfies a pre-defined error bound. For more details about computability of exit-flag functionalities we refer to \cite{Boche2022ExitFlag}.
\end{itemize}
The above discussion shows that iterative algorithms need to rely on a stopping criterion based on heuristics. Although a chosen criterion may be well-adapted to a certain application, we can in general not expect the output of an iterative algorithm to come with guaranteed error bounds. We summarize our conclusions in the following statement, which is a direct implication of \Cref{thm:nonApprox}.
\begin{Corollary}\label{cor:ItAlg}
    There does not exist an effective stopping criterion on Turing machines for iterative algorithms applied to solving inverse problems via basis pursuit \eqref{eq:sparseprob} or lasso \eqref{eq:lasso}. In particular, using the notation introduced in \eqref{eq:ItAlg} we have: There does not exist a Turing machine that, given computable $I$ and a computable error parameter $\gamma>0$ as input, computes an initialization $x^{(0)}$ and a set of parameters $\mu$ such that the iterative process \eqref{eq:ItAlg} is stopped after finitely many steps $m$ with 
    \begin{equation}\label{eq:approxGamma}
        \nnorm{x^\ast - x^{(m)}} \leq \gamma. 
    \end{equation}
    In other words, there does not exist a Turing machine that, given $I$ and $\gamma$, computes an initialization $x^{(0)}$, a set of parameters $\mu$, and an index $m$ such that \eqref{eq:approxGamma} holds.
\end{Corollary}
We can further refine \Cref{cor:ItAlg} by also taking the lower bounds in \Cref{thm:nonApprox} into consideration. Then, the non-existence of the sought Turing machines for specific values of the error parameter $\gamma$ can be explicitly stated.

\subsection{Summary and Outlook}
We have shown that the computational problems posed by the mappings $\Xi_{\text{bp}, \varepsilon}$ and $\Xi_{\text{la}, \lambda}$, i.e., computing the solution(s) of basis pursuit \eqref{eq:sparseprob} and lasso \eqref{eq:lasso}, are not algorithmically solvable in Banach-Mazur sense for a relevant set of relaxation parameters. In other words, any single-valued restriction of the solution mappings --- i.e., fixing an arbitrary element of the possibly many solutions as the function value for each input --- is not Banach-Mazur computable. More importantly, we also exclude the possibility of algorithmically approximate the solution map in Banach-Mazur sense. We provide a negative answer and specify a lower bound on the achievable approximation accuracy (also known as the breakdown-epsilon \cite{bastounis21extended}). Here, the found limitations are not connected to ill-conditioning but do occur on well-conditioned inputs. 
Thus, another characterization of the best-case scenario in practical applications on digital hardware is provided, adding to \cite{bastounis21extended,colbrook21stable}. Moreover, the proof techniques can, in principle, be extended to other formulations as well. This is further evidence, reinforcing the findings in \cite{bastounis21extended,colbrook21stable}, that the computational limitations of solving inverse problems on digital hardware persist through various problem descriptions.

Additionally, the impossibility results also imply non-computability in Borel-Turing sense and thereby match some of the results in \cite{bastounis21extended,colbrook21stable} from a different starting point.
Borel-Turing computability is an appropriate notion to describe certifiably correct computations on digital hardware in the following sense: Given approximations of the exact input of arbitrary accuracy, a (digital) computer calculates approximations of the exact output such that the accuracy increases with the input accuracy. Even more, the computer also provides error bounds for each computed output, i.e., its worst-case distance to the exact solution. The non-computability and non-approximability results imply that there is no algorithm that is able to compute or approximate the sought minimizer of the problem in the described manner.

The algorithmic non-approximability result also has severe implications on the training of neural networks aspired to solve inverse problems via basis pursuit \eqref{eq:sparseprob} and lasso \eqref{eq:lasso} optimization. Given a finite set of samples $\tau_{P}$ as defined in \eqref{eq:tauSparse}, where $P$ is either \eqref{eq:sparseprob} or \eqref{eq:lasso}, the objective is to obtain a mapping that takes as input the measurements and the sampling operator of some unknown image and returns (some approximation of) the image as output. However, this mapping is not computable on digital machines. Therefore, there does not exist a general learning algorithm applicable to arbitrary inverse problems in this setting. In particular, there does not exist an algorithm that produces a neural network $\Phi_{A,\tau_{P}}$, which approximates a single-valued solution mapping with arbitrary precision, based on a training set with computable elements of the form $(A,\tau_{P})$. The crucial point is that we  either have to be content with a precision above the non-approximability threshold, i.e., the breakdown epsilon \cite{bastounis21extended}, or we can not algorithmically check to which degree the algorithmic computation of $\Phi_{A,\tau_{P}}$ succeeds. 

Therefore, in the latter case we can not guarantee the correctness of the output of $\Phi_{A,\tau_{P}}$. Although the output is not necessarily 'wrong' --- it may be very close to the correct reconstruction --- we have no way to algorithmically assess its correctness or deviation since we have no algorithmically constructible measure of distance to the correct reconstruction. Nevertheless, we may still apply $\Phi_{A,\tau_{P}}$ in practice but we have to be aware of its limitations. In particular, the algorithmic non-approximability of the reconstruction map implies that $\Phi_{A,\tau_{P}}$ is guaranteed to fail for specific inputs. Even more, in \cite{colbrook21stable}, for any $K>2$, inputs sets are constructed for which any learning algorithm will produce neural networks that approximate the corresponding solutions with at most $K-1$ correct digits. A possible work-around would be to characterize large classes of (computable) inputs of the inverse problems such that the corresponding reconstructions can be algorithmically computed. Ideally, the user or $\Phi_{A,\tau_{P}}$ itself should algorithmically recognize whether the given data allows for a successful computation which satisfies the specified error bound. Unfortunately, automating such an exit-flag functionality on Turing machines is not feasible for inverse problems \cite{bastounis21extended}.

Finally, we want to stress that in our analysis via Banach-Mazur computability the limitations of algorithmically solving inverse problems are inherently connected to digital hardware. In other computation models the circumstances may be different, as shown in \cite{Boche2022InvProb} for the BSS model describing (noise-free) analog computations. Although digital hardware is prevalent in basically any field of scientific computing, this paradigm may shift in the future. In deep learning, the emergence of neuromorphic hardware \cite{Mead90NeuroComp, schuman17survey} --- a combination of digital and analog computations where elements of a computer are modeled after systems in the human brain and nervous system --- potentially offers a solution to overcome the computational barriers on digital hardware.

\bibliographystyle{IEEEtran}
\bibliography{IEEEabrv,main}

\begin{thebibliography}{10}
\providecommand{\url}[1]{#1}
\csname url@samestyle\endcsname
\providecommand{\newblock}{\relax}
\providecommand{\bibinfo}[2]{#2}
\providecommand{\BIBentrySTDinterwordspacing}{\spaceskip=0pt\relax}
\providecommand{\BIBentryALTinterwordstretchfactor}{4}
\providecommand{\BIBentryALTinterwordspacing}{\spaceskip=\fontdimen2\font plus
\BIBentryALTinterwordstretchfactor\fontdimen3\font minus \fontdimen4\font\relax}
\providecommand{\BIBforeignlanguage}[2]{{%
\expandafter\ifx\csname l@#1\endcsname\relax
\typeout{** WARNING: IEEEtran.bst: No hyphenation pattern has been}%
\typeout{** loaded for the language `#1'. Using the pattern for}%
\typeout{** the default language instead.}%
\else
\language=\csname l@#1\endcsname
\fi
#2}}
\providecommand{\BIBdecl}{\relax}
\BIBdecl

\bibitem{poonen14}
B.~Poonen, \emph{Undecidable problems: a sampler}.\hskip 1em plus 0.5em minus 0.4em\relax Cambridge University Press, 2014, p. 211–241.

\bibitem{Hilbert00Problems}
D.~Hilbert, ``{Mathematical problems},'' \emph{Bulletin of the American Mathematical Society}, vol.~8, no.~10, pp. 437 -- 479, 1902.

\bibitem{Matiyasevich70Diophantine}
Y.~V. Matiyasevich, ``Enumerable sets are diophantine,'' \emph{Soviet Mathematics}, vol.~11, no.~2, pp. 354 -- 357, 1970.

\bibitem{McCulloch43NNs}
W.~S. McCulloch and W.~Pitts, ``A logical calculus of the ideas immanent in nervous activity,'' \emph{The Bulletin of Mathematical Biophysics}, vol.~5, no.~4, pp. 115--133, 1943.

\bibitem{He2015DelvingDI}
K.~He, X.~Zhang, S.~Ren, and J.~Sun, ``Delving {D}eep into {R}ectifiers: {S}urpassing {H}uman-{L}evel {P}erformance on {I}mage{N}et {C}lassification,'' \emph{2015 IEEE International Conference on Computer Vision (ICCV)}, pp. 1026--1034, 2015.

\bibitem{Silver16Go}
D.~Silver, A.~Huang, C.~J. Maddison, A.~Guez, L.~Sifre, G.~van~den Driessche, J.~Schrittwieser, I.~Antonoglou, V.~Panneershelvam, M.~Lanctot, S.~Dieleman, D.~Grewe, J.~Nham, N.~Kalchbrenner, I.~Sutskever, T.~Lillicrap, M.~Leach, K.~Kavukcuoglu, T.~Graepel, and D.~Hassabis, ``Mastering the game of {G}o with deep neural networks and tree search,'' \emph{Nature}, vol. 529, pp. 484--503, 2016.

\bibitem{Brown20GPT3}
T.~Brown, B.~Mann, N.~Ryder, M.~Subbiah, J.~D. Kaplan, P.~Dhariwal, A.~Neelakantan, P.~Shyam, G.~Sastry, A.~Askell, S.~Agarwal, A.~Herbert-Voss, G.~Krueger, T.~Henighan, R.~Child, A.~Ramesh, D.~Ziegler, J.~Wu, C.~Winter, C.~Hesse, M.~Chen, E.~Sigler, M.~Litwin, S.~Gray, B.~Chess, J.~Clark, C.~Berner, S.~McCandlish, A.~Radford, I.~Sutskever, and D.~Amodei, ``Language {M}odels are {F}ew-{S}hot {L}earners,'' in \emph{Advances in Neural Information Processing Systems}, H.~Larochelle, M.~Ranzato, R.~Hadsell, M.~F. Balcan, and H.~Lin, Eds., vol.~33.\hskip 1em plus 0.5em minus 0.4em\relax Curran Associates, Inc., 2020, pp. 1877--1901.

\bibitem{Senior20DeepFold}
A.~W. Senior, R.~Evans, J.~Jumper, J.~Kirkpatrick, L.~Sifre, T.~Green, C.~Qin, A.~Žídek, A.~W.~R. Nelson, A.~Bridgland, H.~Penedones, S.~Petersen, K.~Simonyan, S.~Crossan, P.~Kohli, D.~T. Jones, D.~Silver, K.~Kavukcuoglu, and D.~Hassabis, ``Improved protein structure prediction using potentials from deep learning,'' \emph{Nature}, vol. 577, pp. 706--710, 2020.

\bibitem{LeCun15DL}
Y.~LeCun, Y.~Bengio, and G.~Hinton, ``Deep learning,'' \emph{Nature}, vol. 521, pp. 436--444, 2015.

\bibitem{Rumelhart86BP}
D.~E. Rumelhart, G.~E. Hinton, and R.~J. Williams, ``Learning representations by back-propagating errors,'' \emph{Nature}, vol. 323, pp. 533--536, 1986.

\bibitem{Zu18AutoMap}
B.~Zhu, J.~Z. Liu, S.~F. Cauley, B.~R. Rosen, and M.~S. Rosen, ``Image reconstruction by domain-transform manifold learning,'' \emph{Nature}, vol. 555, pp. 487--492, 2018.

\bibitem{Arridge2019SolvingIP}
S.~R. Arridge, P.~Maass, O.~{\"O}ktem, and C.-B. Sch{\"o}nlieb, ``Solving inverse problems using data-driven models,'' \emph{Acta Numerica}, vol.~28, pp. 1 -- 174, 2019.

\bibitem{Bubba19Shearlet}
T.~A. Bubba, G.~Kutyniok, M.~Lassas, M.~M{\"a}rz, W.~Samek, S.~Siltanen, and V.~Srinivasan, ``Learning the {I}nvisible: A {H}ybrid {D}eep {L}earning-{S}hearlet {F}ramework for {L}imited {A}ngle {C}omputed {T}omography,'' \emph{Inverse Problems}, vol.~35, no.~6, May 2019.

\bibitem{Yang16MRIDL}
Y.~Yang, J.~Sun, H.~Li, and Z.~Xu, ``Deep {ADMM}-{N}et for {C}ompressive {S}ensing {MRI},'' in \emph{Advances in Neural Information Processing Systems}, D.~Lee, M.~Sugiyama, U.~Luxburg, I.~Guyon, and R.~Garnett, Eds., vol.~29.\hskip 1em plus 0.5em minus 0.4em\relax Curran Associates, Inc., 2016.

\bibitem{Hammernik18MRIDL2}
K.~Hammernik, T.~Klatzer, E.~Kobler, M.~P. Recht, D.~K. Sodickson, T.~Pock, and F.~Knoll, ``Learning a variational network for reconstruction of accelerated {MRI} data,'' \emph{Magnetic Resonance in Medicine}, vol.~79, no.~6, pp. 3055--3071, 2018.

\bibitem{Chen2018LowLightPhoto}
C.~Chen, Q.~Chen, J.~Xu, and V.~Koltun, ``Learning to {S}ee in the {D}ark,'' \emph{arXiv:1805.01934}, 2018.

\bibitem{Rivenson17DLmicroscopy}
Y.~Rivenson, Z.~G\"{o}r\"{o}cs, H.~G\"{u}naydin, Y.~Zhang, H.~Wang, and A.~Ozcan, ``Deep learning microscopy,'' \emph{Optica}, vol.~4, no.~11, pp. 1437--1443, Nov 2017.

\bibitem{Araya18DLtomography}
M.~Araya-Polo, J.~Jennings, A.~Adler, and T.~Dahlke, ``Deep-learning tomography,'' \emph{The Leading Edge}, vol.~37, no.~1, pp. 58--66, 2018.

\bibitem{Boche20LTI}
H.~Boche and V.~Pohl, ``Turing {M}eets {C}ircuit {T}heory: Not {E}very {C}ontinuous-{T}ime {LTI} {S}ystem {C}an be {S}imulated on a {D}igital {C}omputer,'' \emph{IEEE Transactions on Circuits and Systems I: Regular Papers}, vol.~67, no.~12, pp. 5051--5064, 2020.

\bibitem{Weihrauch02WaveProp}
K.~Weihrauch and N.~Zhong, ``Is wave propagation computable or can wave computers beat the {T}uring machine?'' \emph{Proceedings of the London Mathematical Society}, vol.~85, pp. 312 -- 332, Sep 2002.

\bibitem{Zhong14EffConvDE}
S.-M. Sun and N.~Zhong, ``On {E}ffective {C}onvergence of {N}umerical {S}olutions for {D}ifferential {E}quations,'' \emph{ACM Trans. Comput. Theory}, vol.~6, no.~1, Mar 2014.

\bibitem{Graca21CompDE}
D.~S. Gra\c{c}a and N.~Zhong, \emph{Computability of {D}ifferential {E}quations}.\hskip 1em plus 0.5em minus 0.4em\relax Cham: Springer International Publishing, 2021, pp. 71--99.

\bibitem{Fettweis20226G}
G.~P. Fettweis and H.~Boche, ``On 6{G} and {T}rustworthiness,'' \emph{Commun. ACM}, vol.~65, no.~4, p. 48–49, Mar 2022.

\bibitem{Turing36Entscheidung}
A.~M. Turing, ``On {C}omputable {N}umbers, with an {A}pplication to the {E}ntscheidungsproblem,'' \emph{Proceedings of the London Mathematical Society}, vol. s2-42, no.~1, pp. 230--265, 1936.

\bibitem{PourEl97WaveEq}
M.~B. Pour-El and N.~Zhong, ``The {W}ave {E}quation with {C}omputable {I}nitial {D}ata {W}hose {U}nique {S}olution {I}s {N}owhere {C}omputable,'' \emph{Mathematical Logic Quarterly}, vol.~43, no.~4, pp. 499--509, 1997.

\bibitem{Elkouss2018MemoryEC}
D.~Elkouss and D.~P{\'e}rez-Garc{\'i}a, ``Memory effects can make the transmission capability of a communication channel uncomputable,'' \emph{Nature Communications}, vol.~9, no.~1, Mar 2018.

\bibitem{Schaefer2019TuringMS}
R.~F. Schaefer, H.~Boche, and H.~V. Poor, ``Turing {M}eets {S}hannon: On the {A}lgorithmic {C}omputability of the {C}apacities of {S}ecure {C}ommunication {S}ystems ({I}nvited {P}aper),'' \emph{2019 IEEE 20th International Workshop on Signal Processing Advances in Wireless Communications (SPAWC)}, pp. 1--5, 2019.

\bibitem{Boche20SpecFac}
H.~Boche and V.~Pohl, ``On the {A}lgorithmic {S}olvability of {S}pectral {F}actorization and {A}pplications,'' \emph{IEEE Transactions on Information Theory}, vol.~66, no.~7, pp. 4574--4592, 2020.

\bibitem{Boche20BandlimitedSignals}
H.~Boche and U.~J. Mönich, ``Turing {C}omputability of {F}ourier {T}ransforms of {B}andlimited and {D}iscrete {S}ignals,'' \emph{IEEE Transactions on Signal Processing}, vol.~68, pp. 532--547, 2020.

\bibitem{Szegedy14AdvEx}
C.~Szegedy, W.~Zaremba, I.~Sutskever, J.~Bruna, D.~Erhan, I.~Goodfellow, and R.~Fergus, ``Intriguing properties of neural networks,'' in \emph{2nd International Conference on Learning Representations, {ICLR} 2014, Conference Track Proceedings}, Y.~Bengio and Y.~LeCun, Eds., 2014.

\bibitem{Antun2020InstabilitiesDL}
V.~Antun, F.~Renna, C.~Poon, B.~Adcock, and A.~C. Hansen, ``On instabilities of deep learning in image reconstruction and the potential costs of {AI},'' \emph{Proceedings of the National Academy of Sciences}, vol. 117, no.~48, pp. 30\,088--30\,095, 2020.

\bibitem{Xie20XAI}
N.~Xie, G.~Ras, M.~van Gerven, and D.~Doran, ``Explainable {D}eep {L}earning: {A} {F}ield {G}uide for the {U}ninitiated,'' \emph{arXiv:2004.14545}, 2020.

\bibitem{Kolek2021ratedistortion}
S.~Kolek, D.~A. Nguyen, R.~Levie, J.~Bruna, and G.~Kutyniok, \emph{A {R}ate-{D}istortion {F}ramework for {E}xplaining {B}lack-box {M}odel {D}ecisions}.\hskip 1em plus 0.5em minus 0.4em\relax Cham: Springer International Publishing, 2022, pp. 91--115.

\bibitem{Adcock20gap}
B.~Adcock and N.~Dexter, ``The {G}ap between {T}heory and {P}ractice in {F}unction {A}pproximation with {D}eep {N}eural {N}etworks,'' \emph{SIAM Journal on Mathematics of Data Science}, vol.~3, no.~2, pp. 624--655, 2021.

\bibitem{Berner2021modernMathDL}
J.~Berner, P.~Grohs, G.~Kutyniok, and P.~Petersen, ``The {M}odern {M}athematics of {D}eep {L}earning,'' \emph{arXiv:2105.04026}, 2021.

\bibitem{colbrook21stable}
M.~J. Colbrook, V.~Antun, and A.~C. Hansen, ``The difficulty of computing stable and accurate neural networks: {O}n the barriers of deep learning and {S}male’s 18th problem,'' \emph{Proceedings of the National Academy of Sciences}, vol. 119, no.~12, 2022.

\bibitem{gottschling20troublesome}
N.~M. Gottschling, V.~Antun, B.~Adcock, and A.~C. Hansen, ``The troublesome kernel: why deep learning for inverse problems is typically unstable,'' \emph{arXiv:2001.01258}, 2020.

\bibitem{bastounis21extended}
A.~Bastounis, A.~C. Hansen, and V.~Vlačić, ``The extended {S}male's 9th problem -- {O}n computational barriers and paradoxes in estimation, regularisation, computer-assisted proofs and learning,'' \emph{arXiv:2110.15734}, 2021.

\bibitem{bastounis2021Classification}
------, ``The mathematics of adversarial attacks in {AI} -- {W}hy deep learning is unstable despite the existence of stable neural networks,'' \emph{arXiv:2109.06098}, 2021.

\bibitem{Ben2015SCI}
J.~Ben-Artzi, M.~J. Colbrook, A.~C. Hansen, O.~Nevanlinna, and M.~Seidel, ``Computing {S}pectra -- {O}n the {S}olvability {C}omplexity {I}ndex {H}ierarchy and {T}owers of {A}lgorithms,'' \emph{arXiv:1508.03280}, 2015.

\bibitem{Blum89BSS}
L.~Blum, M.~Shub, and S.~Smale, ``{On a theory of computation and complexity over the real numbers: $NP$- completeness, recursive functions and universal machines},'' \emph{Bulletin (New Series) of the American Mathematical Society}, vol.~21, no.~1, pp. 1 -- 46, 1989.

\bibitem{Boche2022InvProb}
H.~Boche, A.~Fono, and G.~Kutyniok, ``{I}nverse {P}roblems {A}re {S}olvable on {R}eal {N}umber {S}ignal {P}rocessing {H}ardware,'' \emph{arXiv:2204.02066v2}, 2022.

\bibitem{Candes06Stable}
E.~J. Candès, J.~K. Romberg, and T.~Tao, ``Stable {S}ignal {R}ecovery from {I}ncomplete and {I}naccurate {M}easurements,'' \emph{Communications on Pure and Applied Mathematics}, vol.~59, no.~8, pp. 1207--1223, 2006.

\bibitem{Tropp06Relax}
J.~Tropp, ``Just relax: convex programming methods for identifying sparse signals in noise,'' \emph{IEEE Transactions on Information Theory}, vol.~52, no.~3, pp. 1030--1051, 2006.

\bibitem{Belloni11SquareRootLasso}
A.~Belloni, V.~Chernozhukov, and L.~Wang, ``Square-{R}oot {L}asso: {P}ivotal {R}ecovery of {S}parse {S}ignals via {C}onic {P}rogramming,'' \emph{Biometrika}, vol.~98, no.~4, pp. 791--806, Dec 2011.

\bibitem{Akhtar18ThreatAdvAtt}
N.~Akhtar and A.~Mian, ``Threat of {A}dversarial {A}ttacks on {D}eep {L}earning in {C}omputer {V}ision: A {S}urvey,'' \emph{IEEE Access}, vol.~6, pp. 14\,410--14\,430, 2018.

\bibitem{Carlini18AudioAdvEx}
N.~Carlini and D.~Wagner, ``Audio {A}dversarial {E}xamples: {T}argeted {A}ttacks on {S}peech-to-{T}ext,'' in \emph{2018 IEEE Security and Privacy Workshops (SPW)}, 2018, pp. 1--7.

\bibitem{Moosavi16DeepFool}
S.-M. Moosavi-Dezfooli, A.~Fawzi, and P.~Frossard, ``Deepfool: A {S}imple and {A}ccurate {M}ethod to {F}ool {D}eep {N}eural {N}etworks,'' in \emph{2016 IEEE Conference on Computer Vision and Pattern Recognition (CVPR)}, 2016, pp. 2574--2582.

\bibitem{Finlayson19AdvAttMed}
S.~G. Finlayson, J.~D. Bowers, J.~Ito, J.~L. Zittrain, A.~L. Beam, and I.~S. Kohane, ``Adversarial attacks on medical machine learning,'' \emph{Science}, vol. 363, no. 6433, pp. 1287--1289, 2019.

\bibitem{Madry18AdvTraining}
A.~Madry, A.~Makelov, L.~Schmidt, D.~Tsipras, and A.~Vladu, ``Towards {D}eep {L}earning {M}odels {R}esistant to {A}dversarial {A}ttacks,'' in \emph{6th International Conference on Learning Representations, {ICLR} 2018, Conference Track Proceedings}, 2018.

\bibitem{Papernot16Distillation}
N.~Papernot, P.~McDaniel, X.~Wu, S.~Jha, and A.~Swami, ``Distillation as a {D}efense to {A}dversarial {P}erturbations {A}gainst {D}eep {N}eural {N}etworks,'' in \emph{2016 IEEE Symposium on Security and Privacy (SP)}.\hskip 1em plus 0.5em minus 0.4em\relax Los Alamitos, CA, USA: IEEE Computer Society, May 2016, pp. 582--597.

\bibitem{Ilyas19AdvExNotBugs}
A.~Ilyas, S.~Santurkar, D.~Tsipras, L.~Engstrom, B.~Tran, and A.~Madry, \emph{Adversarial {E}xamples {A}re {N}ot {B}ugs, {T}hey {A}re {F}eatures}.\hskip 1em plus 0.5em minus 0.4em\relax Red Hook, NY, USA: Curran Associates Inc., 2019.

\bibitem{Tsipras18RobustnessOdds}
D.~Tsipras, S.~Santurkar, L.~Engstrom, A.~Turner, and A.~Madry, ``Robustness {M}ay {B}e at {O}dds with {A}ccuracy,'' in \emph{7th International Conference on Learning Representations}, 2019.

\bibitem{gilmer2018adversarial}
J.~Gilmer, L.~Metz, F.~Faghri, S.~S. Schoenholz, M.~Raghu, M.~Wattenberg, and I.~Goodfellow, ``Adversarial {S}pheres,'' \emph{arXiv:1801.02774v3}, 2018.

\bibitem{Mead90NeuroComp}
C.~Mead, ``Neuromorphic electronic systems,'' \emph{Proceedings of the IEEE}, vol.~78, no.~10, pp. 1629--1636, 1990.

\bibitem{schuman17survey}
C.~D. Schuman, T.~E. Potok, R.~M. Patton, J.~D. Birdwell, M.~E. Dean, G.~S. Rose, and J.~S. Plank, ``A {S}urvey of {N}euromorphic {C}omputing and {N}eural {N}etworks in {H}ardware,'' \emph{arXiv:1705.06963}, 2017.

\bibitem{vonNeumann45Architektur}
J.~von Neumann, ``First {D}raft of a {R}eport on the {EDVAC},'' \emph{IEEE Annals of the History of Computing}, vol.~15, no.~4, pp. 27--75, 1993.

\bibitem{Goedel31Unentscheidbar}
K.~Gödel, ``{\"U}ber formal unentscheidbare {S}ätze der {P}rincipia {M}athematica und verwandter {S}ysteme,'' \emph{Monatshefte für Mathematik und Physik}, vol.~38, no.~1, pp. 173--198, 1931.

\bibitem{Myhill71Noncomp}
J.~Myhill, ``A recursive function, defined on a compact interval and having a continuous derivative that is not recursive.'' \emph{Michigan Mathematical Journal}, vol.~18, no.~2, pp. 97 -- 98, 1971.

\bibitem{Soare87RecursivelyES}
R.~I. Soare, ``Recursively enumerable sets and degrees,'' \emph{Bulletin of the American Mathematical Society}, vol.~84, pp. 1149--1181, 1987.

\bibitem{Weihrauch00CompAnal}
K.~Weihrauch, \emph{Computable Analysis: An Introduction}.\hskip 1em plus 0.5em minus 0.4em\relax Berlin, Heidelberg: Springer-Verlag, 2000.

\bibitem{Pour-El17Computability}
M.~B. Pour-El and J.~I. Richards, \emph{Computability in {A}nalysis and {P}hysics}, ser. Perspectives in Logic.\hskip 1em plus 0.5em minus 0.4em\relax Cambridge University Press, 2017.

\bibitem{AvigadBrattka14CompAnal}
J.~Avigad and V.~Brattka, \emph{Computability and analysis: the legacy of {A}lan {T}uring}, ser. Lecture Notes in Logic.\hskip 1em plus 0.5em minus 0.4em\relax Cambridge University Press, 2014, p. 1–47.

\bibitem{Kleene36Recursive}
S.~Kleene, ``General recursive functions of natural numbers.'' \emph{Mathematische Annalen}, vol. 112, pp. 727--742, 1936.

\bibitem{Turing37Equivalence}
A.~M. Turing, ``Computability and lambda-{D}efinability,'' \emph{Journal of Symbolic Logic}, vol.~2, no.~4, p. 153–163, 1937.

\bibitem{Brattka2016CompAnalHistory}
V.~Brattka, ``{C}omputability and {A}nalysis, a {H}istorical {A}pproach,'' in \emph{Pursuit of the Universal}.\hskip 1em plus 0.5em minus 0.4em\relax Springer International Publishing, 2016, pp. 45--57.

\bibitem{Boche2020Fekete}
H.~Boche, Y.~Böck, and C.~Deppe, ``On {E}ffective {C}onvergence in {F}ekete's {L}emma and {R}elated {C}ombinatorial {P}roblems in {I}nformation {T}heory,'' \emph{arXiv:2010.09896}, 2020.

\bibitem{Boche2022PseudoInverse}
H.~Boche, A.~Fono, and G.~Kutyniok, ``Non-{C}omputability of the {P}seudoinverse on {D}igital {C}omputers,'' \emph{arXiv:2212.02940}, 2022.

\bibitem{Daubechies04SparseReg}
I.~Daubechies, M.~Defrise, and C.~De~Mol, ``An iterative thresholding algorithm for linear inverse problems with a sparsity constraint,'' \emph{Communications on Pure and Applied Mathematics}, vol.~57, no.~11, pp. 1413--1457, 2004.

\bibitem{Candes06RobUnc}
E.~Candes, J.~Romberg, and T.~Tao, ``Robust uncertainty principles: exact signal reconstruction from highly incomplete frequency information,'' \emph{IEEE Transactions on Information Theory}, vol.~52, no.~2, pp. 489--509, 2006.

\bibitem{Donoho06CompSens}
D.~Donoho, ``Compressed sensing,'' \emph{IEEE Transactions on Information Theory}, vol.~52, no.~4, pp. 1289--1306, 2006.

\bibitem{Candes06UnivEncStrat}
E.~J. Candes and T.~Tao, ``Near-{O}ptimal {S}ignal {R}ecovery {F}rom {R}andom {P}rojections: {U}niversal {E}ncoding {S}trategies?'' \emph{IEEE Transactions on Information Theory}, vol.~52, no.~12, pp. 5406--5425, 2006.

\bibitem{Candes05DecLP}
E.~Candes and T.~Tao, ``Decoding by linear programming,'' \emph{IEEE Transactions on Information Theory}, vol.~51, no.~12, pp. 4203--4215, 2005.

\bibitem{Schlemper18CNNvsCS}
J.~Schlemper, J.~Caballero, J.~V. Hajnal, A.~N. Price, and D.~Rueckert, ``A {D}eep {C}ascade of {C}onvolutional {N}eural {N}etworks for {D}ynamic {MR} {I}mage {R}econstruction,'' \emph{IEEE Transactions on Medical Imaging}, vol.~37, no.~2, pp. 491--503, 2018.

\bibitem{Jin17DCNNInvProb}
K.~H. Jin, M.~T. McCann, E.~Froustey, and M.~Unser, ``Deep {C}onvolutional {N}eural {N}etwork for {I}nverse {P}roblems in {I}maging,'' \emph{IEEE Transactions on Image Processing}, vol.~26, no.~9, pp. 4509--4522, 2017.

\bibitem{Adler17DNNInvProb}
J.~Adler and O.~Öktem, ``Solving ill-posed inverse problems using iterative deep neural networks,'' \emph{Inverse Problems}, vol.~33, no.~12, p. 124007, Nov 2017.

\bibitem{Ongie20DLInvProb}
G.~Ongie, A.~Jalal, C.~A. Metzler, R.~G. Baraniuk, A.~G. Dimakis, and R.~Willett, ``Deep {L}earning {T}echniques for {I}nverse {P}roblems in {I}maging,'' \emph{IEEE Journal on Selected Areas in Information Theory}, vol.~1, no.~1, pp. 39--56, 2020.

\bibitem{Mousavi15DLvCSSigRec}
A.~Mousavi, A.~B. Patel, and R.~G. Baraniuk, ``A deep learning approach to structured signal recovery,'' in \emph{2015 53rd Annual Allerton Conference on Communication, Control, and Computing (Allerton)}, 2015, pp. 1336--1343.

\bibitem{Goodfellow16DL}
I.~Goodfellow, Y.~Bengio, and A.~Courville, \emph{Deep Learning}.\hskip 1em plus 0.5em minus 0.4em\relax MIT Press, 2016, \url{http://www.deeplearningbook.org}.

\bibitem{Boche2020SmeetsT}
H.~Boche, R.~F. Schaefer, and H.~V. Poor, ``{S}hannon meets {T}uring: {N}on-{C}omputability and {N}on-{A}pproximability of the {F}inite {S}tate {C}hannel {C}apacity,'' \emph{Communications in Information and Systems}, vol.~20, no.~2, pp. 81--116, 2020.

\bibitem{Blum98ComplRealComp}
L.~Blum, F.~Cucker, M.~Shub, and S.~Smale, \emph{Complexity and Real Computation}.\hskip 1em plus 0.5em minus 0.4em\relax New York: Springer Verlag, 1998.

\bibitem{Bürgisser2013condition}
P.~B{\"u}rgisser and F.~Cucker, \emph{Condition:\! The Geometry of Numerical Algorithms}, ser. \selectlanguage{ngerman}Grundlehren der mathematischen Wissenschaften.\hskip 1em plus 0.5em minus 0.4em\relax Springer Berlin Heidelberg, 2013.

\bibitem{Candes2011Nesta}
S.~Becker, J.~Bobin, and E.~J. Cand\`{e}s, ``{NESTA}: {A} {F}ast and {A}ccurate {F}irst-{O}rder {M}ethod for {S}parse {R}ecovery,'' \emph{SIAM Journal on Imaging Sciences}, vol.~4, no.~1, pp. 1--39, 2011.

\bibitem{Eldar2021ComprImag}
O.~Drori, A.~Mamistvalov, O.~Solomon, and Y.~C. Eldar, ``Compressed {U}ltrasound {I}maging: {F}rom {S}ub-{N}yquist {R}ates to {S}uper {R}esolution,'' \emph{IEEE BITS the Information Theory Magazine}, vol.~1, no.~1, pp. 27--44, 2021.

\bibitem{LeCun2010Lista}
K.~Gregor and Y.~LeCun, ``Learning {F}ast {A}pproximations of {S}parse {C}oding,'' in \emph{Proceedings of the 27th International Conference on International Conference on Machine Learning}, ser. ICML'10.\hskip 1em plus 0.5em minus 0.4em\relax Madison, WI, USA: Omnipress, 2010, p. 399–406.

\bibitem{Eldar2021AlgUnroll}
V.~Monga, Y.~Li, and Y.~C. Eldar, ``Algorithm {U}nrolling: {I}nterpretable, {E}fficient {D}eep {L}earning for {S}ignal and {I}mage {P}rocessing,'' \emph{IEEE Signal Processing Magazine}, vol.~38, no.~2, pp. 18--44, 2021.

\bibitem{Bredies2008LinConv}
K.~Bredies and D.~Lorenz, ``Linear {C}onvergence of {I}terative {S}oft-{T}hresholding,'' \emph{Journal of Fourier Analysis and Applications volume}, vol.~14, no.~4, pp. 813–--837, 2008.

\bibitem{Beck2009FISTA}
A.~Beck and M.~Teboulle, ``A {F}ast {I}terative {S}hrinkage-{T}hresholding {A}lgorithm for {L}inear {I}nverse {P}roblems,'' \emph{SIAM Journal on Imaging Sciences}, vol.~2, no.~1, pp. 183--202, 2009.

\bibitem{Boche2022ExitFlag}
H.~Boche and V.~Pohl, ``On {N}on-{D}etectability of {N}on-{C}omputability and the {D}egree of {N}on-{C}omputability of {S}olutions of {C}ircuit and {W}ave {E}quations on {D}igital {C}omputers,'' \emph{IEEE Transactions on Information Theory}, vol.~68, no.~8, pp. 5561--5578, 2022.

\end{thebibliography}

\vfill

\end{document}